\documentclass[preprint,12pt]{elsarticle}

\usepackage{amssymb}

\usepackage{graphicx}

\usepackage[utf8]{inputenc}

\usepackage{hyperref}
\usepackage{url}
\usepackage{amsfonts}
\usepackage{multirow}
\usepackage{amsmath}
\usepackage{amsthm}

\usepackage{pdflscape}
\usepackage{bm}
\usepackage{indentfirst}
\usepackage{physics}
\usepackage[algoruled,boxed,lined]{algorithm2e}

\usepackage{natbib}

\usepackage{graphicx}
\usepackage{caption}
\usepackage{subcaption}

\usepackage{amsmath}

\DeclareMathOperator*{\argmin}{arg\,min}

\newtheorem{theorem}{Theorem}[section]
\newtheorem{corollary}{Corollary}[theorem]
\newtheorem{lemma}[theorem]{Lemma}
\newtheorem{definition}{Definition}
\newtheorem{example}{Example}

\usepackage[dvipsnames]{xcolor}

\journal{Pattern Recognition}

\begin{document}

\begin{frontmatter}



\title{{Mixability of Integral Losses:\\a Key to} Efficient Online Aggregation\\of Functional and Probabilistic Forecasts}


\author[sk]{Alexander Korotin}
\author[iitp,sk]{Vladimir V'yugin}
\author[sk]{Evgeny Burnaev}

\address[sk]{Skolkovo Institute of Science and Technology, Moscow, Russia}
\address[iitp]{Institute for Information Transmission Problems, Moscow, Russia}

\begin{abstract}
In this paper we extend the setting of the online prediction with expert advice to function-valued forecasts. At each step of the online game several experts predict a function, and the learner has to efficiently aggregate these functional forecasts into a single forecast. We adapt basic mixable (and exponentially concave) loss functions to compare functional predictions and prove that these adaptations are also mixable (exp-concave). We call this phenomenon {mixability (exp-concavity) of integral loss functions}. As an application of our main result, we prove that various loss functions used for probabilistic forecasting are mixable (exp-concave). The considered losses include Sliced Continuous Ranked Probability Score, Energy-Based Distance, Optimal Transport Costs \& Sliced Wasserstein-2 distance, Beta-2 \& Kullback-Leibler divergences, Characteristic function and Maximum Mean Discrepancies.

\end{abstract}




\end{frontmatter}


\section{Introduction}

Classic online prediction with expert advice \cite{vovk1998game}
is a competition between the learner and the adversary which consists of $T$ sequential steps. At each step of the game the learner has to efficiently aggregate the predictions of a given (fixed) pool of experts (algorithms, machine learning models, human experts). The described scenario is formalized in Protocol \ref{protocol}.

At the beginning of each step $t=1,2,\dots,T$ experts $1,2,3,\dots,N$ from a given (fixed) pool of experts $\mathcal{N}=\{1,2,\dots,N\}$ output their predictions $\gamma_{t}^{n}\in\Gamma$ of yet unknown value $\omega_{t}\in\Omega$. Next, the learner has to combine (merge, mix, average) these forecasts into a single forecast $\overline{\gamma}_{t}\in\Gamma$.

At the end of the step the true outcome is revealed, and both the learner and the experts suffer their losses by using a loss function $\lambda:\Gamma\times \Omega\rightarrow \mathbb{R}_{+}$. The loss of expert $n\in\mathcal{N}$ is denoted by $l_{t}^{n}=\lambda(\gamma_{t}^{n},\omega_{t})$; the loss of the learner is denoted by $h_{t}=\lambda(\overline{\gamma_{t}},\omega_{t})$. We use $H_{t}=\sum_{\tau=1}^{t}h_{\tau}$ and $L_{t}^{N}=\sum_{\tau=1}^{t}l_{\tau}^{n}$ to denote 
the cumulative loss of the learner and the expert $n\in\mathcal{N}$ respectively.
\begin{algorithm}
\SetAlgorithmName{Protocol}{empty}{Empty}
\SetKwInOut{Parameters}{Parameters}
\Parameters{Pool of experts $\mathcal{N}=\{1,2,3\dots,N\}$; Game length $T$; Loss function $\lambda:\Gamma\times \Omega\rightarrow \mathbb{R}_{+}$}
\For{$t=1,2,\dots,T$}{
1. Experts $n\in\mathcal{N}$ provide forecasts $\gamma_{t}^{n}\in \Gamma$\;
2. Learner combines forecasts $\gamma_{t}^{n}$ into forecast $\overline{\gamma_{t}}\in\Gamma$\;
  3. Nature reveals true outcome $\omega_{t}\in\Omega$\;
  4. Experts $n\in\mathcal{N}$ suffer losses $l_{t}^{n}=\lambda(\gamma_{t}^{n},\omega_{t})$\;
  5. Learner suffers loss $h_{t}=\lambda(\overline{\gamma_{t}},\omega_{t})$\;
 }
\caption{Online Prediction with Expert Advice}
\label{protocol}
\end{algorithm}
The goal of the learner is to perform as efficiently as possible w.r.t the best expert\footnote{In the online scenario, the best expert is unknown beforehand.} in the pool, i.e. to minimize the regret

$$R_{T}=\sum_{t=1}^{T}h_{t}-\min_{n\in\mathcal{N}} \sum_{t=1}^{T}l_{t}^{n}=H_{T}-\min_{n\in\mathcal{N}} L_{T}^{n}.$$

Among lots of existing learner's strategies for combining experts predictions \cite{cesa-bianchi,hazan2016introduction},
the aggregating algorithm (AA) by \cite{vovk1998game}
is typically considered to be the best. For a wide class of $\eta$-mixable ($\eta$-exponentially concave) loss functions $\lambda$ it provides the way to efficiently combine experts' predictions $\{\gamma_{t}^{n}\}_{n=1}^{N}$ to a single prediction $\overline{\gamma_{t}}$ so that the learner achieves a small regret bound. More precisely, if the learner follows the AA strategy, the regret w.r.t. the best expert will not exceed $\frac{\ln N}{\eta}$, i.e. $R_{T}\leq \frac{\ln N}{\eta}$. Here $\eta$ is the maximal constant for which the loss function $\lambda$ is $\eta$-mixable ($\eta$-exp-concave). The bound does not depend on the game length $T$. Besides, the knowledge of $T$ is not required before the game. Many existing loss functions $\lambda:\Gamma\times \Omega \rightarrow \mathbb{R}_{+}$ are mixable (exp-concave). Thus, AA can be efficiently applied to combine experts' predictions. Below we briefly review two most common practical online learning problems and two corresponding (mixable) loss functions which are typically used.

The most straightforward practical application of online prediction with expert advice is online \textbf{time-series forecasting} or online \textbf{regression}. Typical examples include electricity price and load forecasting \cite{gaillard2015forecasting},
thermal load forecasting \cite{geysen2018operational}, etc. Predictions $\gamma^{n}\in\Gamma$ and the outcome $\omega\in\Omega$ are usually assumed to be real-valued, i.e. $\Omega, \Gamma\subset \mathbb{R}$. Naturally, the \textbf{squared loss function} $\lambda(\gamma,\omega)=(\gamma-\omega)^{2}$ is used to compare the prediction with the true output. It is known to be mixable \cite{vovk1998game} and exponentially concave \cite{hazan2016introduction} under certain boundness conditions.

Another practical scenario is online \textbf{classification with expert advice}. In this case the goal of the learner is to predict probabilities $\gamma\in\Gamma=\Delta_{K}=\{(\|\gamma\|_{1}=1)\wedge(\gamma>0)\}$ for a given list of events $\Omega=\{1,2,\dots,K\}$ based on experts' predictions. After the true outcome $\omega$ is revealed, the forecast is typically assessed by using the \textbf{logarithmic loss}, i.e. $\lambda(\gamma,\omega)=-\log \gamma(\omega)$. It is mixable and exp-concave, see \cite{vovk1998competitive,adamskiy2012putting}. Such a multi-class classification scenario may, for example, refer to prediction of sports results or the forecasting of precipitation occurrence.
Other classification losses such as \textbf{Brier score} $\lambda(\gamma,\omega)=\sum_{k=1}^{K}(\gamma_{k}-\mathbb{I}[\omega=k])^{2}$ are also known to be mixable, see \cite{vovk2009prediction}. Mixability holds true even if the output is itself a distribution $\omega\in \Omega=\Delta_{K}$ and the score is given by $\lambda(\gamma,\omega)=\|\omega-\gamma\|^{2}$, see \cite{zhdanov2012universal}.

A more advanced problem is the prediction of \textbf{vector-valued outcomes}. In this case both the experts and the learner output a finite-dimensional vector, e.g. a weekly weather forecast.
For example, for $D$-dimensional outputs in regression, it is natural to consider $\Omega:=\Omega^{D}$ and $\Gamma:=\Gamma^{D}$. Naturally, the vector-input square loss
\begin{equation}\lambda_{D}(\gamma,\omega)=\frac{1}{D}\sum_{d=1}^{D}\lambda(\gamma_{d},\omega_{d})=\frac{1}{D}\sum_{d=1}^{D}(\omega_{d}-\gamma_{d})^{2}=\frac{1}{D}\|\omega-\gamma\|^{2}\label{vector-square-loss}\end{equation}
can be used to assess the quality of the forecasts.

Until recently, it was unknown whether \textbf{vectorized loss functions} such as \eqref{vector-square-loss} are mixable and AA can be efficiently applied. {The mixability of vectorized loss functions was studied} by \cite{Kaln2017}. They proved that every vector loss function of the form $\lambda(\omega,\gamma)=\frac{1}{D}\sum_{d=1}^{D}\lambda(\omega_{d},\gamma_{d})$ is $\eta$-mixable (exp-concave) if the corresponding $1$-dimensional-input $\lambda$ is $\eta$-mixable (exp-concave). Meanwhile, the aggregated forecast is built by coordinate-wise aggregation of experts' forecasts.

In this paper, we introduce the general notion of {\textbf{integral loss functions} and prove their \textbf{mixability} (\textbf{exp-concavity})}. We consider the online scenario to predict the function: at each step $t$ the experts output functions $\gamma_{t}^{n}:\mathcal{X}\rightarrow \Gamma$ (i.e. $\gamma_{t}^{n}\in\Gamma^{\mathcal{X}}$) and the learner has to combine these functions into a single function $\overline{\gamma}_{t}:\mathcal{X}\rightarrow \Gamma$. The true output is a function $\omega_{t}:\mathcal{X}\rightarrow \Omega$.
{For example, in the task of probabilistic forecasting of a scalar value, each experts' prediction may be provided as the cumulative distribution function $\mathbb{R}^{1}\rightarrow [0,1]$ (CDF) of predictive distribution. In turn, the true output is the CDF of the empirical observed outcome. We provide other examples related to probabilistic forecasting in Table 1 of Section \ref{sec-table}.
}

For the function-valued forecasting it is reasonable to measure loss via \textbf{integral loss functions} which naturally arise from loss functions used for comparing one-dimensional outcomes.

\begin{definition}[Integral loss function]Let $(\Gamma,\sigma_{\Gamma}),(\Omega,\sigma_{\Omega}),(\mathcal{X},\sigma_{\mathcal{X}})$ be measurable spaces.
Assume that ${\lambda:\Gamma\times\Omega\rightarrow\mathbb{R}_{+}}$ is a loss function measurable w.r.t. $\sigma_{\Gamma}\times\sigma_{\Omega}$. Let $\mu_{\omega}$ be $\omega$-dependent measure on $(\mathcal{X},\sigma_{\mathcal{X}})$ and $u_{\omega}$ be some $\omega$-dependent $\sigma$-finite non-negative measurable function satisfying
$$\int_{\mathcal{X}}u_{\omega}(x)d\mu_{\omega}(x)=1$$
for all $\omega$. Then the function $\lambda_{u,\mu}:\mathcal{M}(\Gamma^{\mathcal{X}})\times\mathcal{M}(\Omega^{\mathcal{X}})\rightarrow\mathbb{R}_{+}$ defined by
\begin{equation}\lambda_{u,\mu}(\gamma,\omega)=\int_{\mathcal{X}}\lambda\big(\gamma(x),\omega(x)\big) u_{\omega}(x)d\mu_{\omega}(x)\label{integral-loss}\end{equation}
is called an $\mathcal{X}$-integral $\lambda$-loss function.\footnote{The usage of weight function $u_{\omega}$  together with measure $\mu_{\omega}$ is redundant. One may naturally eliminate it by changing variables: $(u_{\omega},\mu_{\omega})\mapsto (u_{\omega}',\mu_{\omega}')=(1,\mu_{\omega}')$, with $d\mu_{\omega}':=u_{\omega}d\mu_{\omega}$. However, we keep the notation over-parametrized to be compatible with all the losses discussed in Section \ref{sec-forecasting-probs}.} Here we use $\mathcal{M}(\Gamma^{\mathcal{X}}),\mathcal{M}(\Omega^{\mathcal{X}})$ to denote the sets of all measurable functions $\mathcal{X}\rightarrow\Gamma$ and $\mathcal{X}\rightarrow\Omega$ respectively.
\end{definition}

Clearly, such a general scenario of forecasting under integral loss functions extends vector-valued forecasting scenario of \cite{Kaln2017}. Indeed, if ${|\mathcal{X}|=D}$, one may use $u_{\omega}(x)\equiv 1$ and $\mu_{\omega}(x)\equiv \frac{1}{|\mathcal{X}|}=\frac{1}{D}$ and obtain a vectorized loss.

In real life, function-valued predictions can be used for forecasting physical processes for a period ahead, e.g. temperature distribution \cite{chen2012high}, ocean wave prediction \cite{rusu2013evaluation}, etc. Besides, every \textbf{probabilistic forecast} is actually a function, e.g. density or cumulative distribution function, and classical function-based losses can be used to assess the quality.

\vspace{2mm}\noindent\textbf{The main contributions of the paper are:}
\begin{enumerate}
    \item We introduce the concept of function-valued forecasting and related concept of {mixable (exponentially concave) integral loss functions}.
    \item We prove that for every $\eta$-mixable (exp-concave) measurable loss function ${\lambda:\Gamma\times\Omega\rightarrow \mathbb{R}_{+}}$ its corresponding $\mathcal{X}$-integral $\lambda$-loss function $\lambda_{u,\mu}:\mathcal{M}(\Gamma^{\mathcal{X}})\times \mathcal{M}(\Omega^{\mathcal{X}})\rightarrow \mathbb{R}_{+}$ is $\eta$-mixable (exp-concave) for every admissible  $u,\mu$. The aggregated forecast is built point-wise according to the aggregating rule for $\lambda$.
    \item We demonstrate applications of our results to probabilistic forecasting. We derive mixability (exp-concavity) for Sliced Continuous Ranked Probability Score, Energy-Based Distance, Beta-2 and Kullback-Leibler Divergences, Optimal transport costs \& Sliced Wasserstein-2 Distance, Characteristic Function and Maximum Mean Discrepancies. The results are summarised in Table 1 of Subsection  \ref{sec-table}.
\end{enumerate}

Although our paper is mainly built around online learning framework of prediction with expert advice, we emphasize that the properties of mixability and exponential concavity that we study are extremely useful in other areas of machine learning. In \textbf{statistical machine learning}, mixability typically guarantees faster convergence, see \cite{erven2012mixability}.
In \textbf{online convex optimization} exponential concavity usually leads to better regret bounds, see \cite{hazan2016introduction}. 

\vspace{2mm}\noindent\textbf{The article is structured as follows.} In Section \ref{sec-preliminaries}, we recall the definitions of mixability and exponential concavity of loss functions. In Section \ref{sec-int-mix}, we state the theorem on mixability (exp-concavity) of integral loss functions and prove it. In Section \ref{sec-forecasting-probs}, we apply our result to prove mixability of different loss functions used for comparing probability distributions. The results are summarised in Table 1 of Subsection \ref{sec-table}. In \ref{sec-complex-mix}, we give minor technical details. In \ref{sec-aa}, we review the strategy of AA and recall derivation of algorithm's constant regret bound.

\section{Preliminaries}
\label{sec-preliminaries}

In this section, we recall the definition of mixability and exponential concavity of loss functions.

\begin{definition}[Mixable loss function]A function $\lambda:\Gamma\times\Omega\rightarrow \mathbb{R}$ is called $\eta$-mixable if for all $N=1,2,\dots$, probability vectors $({\alpha}^{1},\dots,{\alpha}^{N})$ and vectors of forecasts $(\gamma^{1},\dots,\gamma^{N})\in \Gamma^{N}$ there exists an aggregated forecast $\overline{\gamma}\in\Gamma$ such that for all $\omega\in\Omega$ the following holds true:

\begin{equation}\exp\big[-\eta\lambda(\overline{\gamma}, \omega)\big]\geq \sum_{n=1}^{N}{\alpha}^{n}\exp\big[-\eta \lambda(\gamma^{n},\omega)\big].
\label{def-formula-mixability}
\end{equation}
\end{definition}
If function $\lambda$ is $\eta$-mixable, then it is also $\eta'$-mixable for all $0<\eta'\leq \eta$. The \textbf{maximal} $\eta$ (for which $\lambda$ is mixable) is always used in order to obtain lower regret bound for AA.


For $\eta$-mixable function $\lambda$ there exists a \textbf{substitution function}
$${\Sigma:\Gamma^{N}\times\Delta_{N}\rightarrow\Gamma}$$
which performs aggregation \eqref{def-formula-mixability} of forecasts $\gamma^{n}$ w.r.t. weights ${\alpha}^{n}$, and outputs aggregated forecast $\overline{\gamma}$. Such a function may be non-unique. For common loss functions, there are usually specific substitution functions  (given by exact formulas) under consideration.

\begin{example}[Square loss]
The function $\lambda(\gamma, \omega)=(\gamma-\omega)^{2}$ with $\Omega=\Gamma=[l,r]$ is $\frac{2}{(r-l)^{2}}$-mixable, see \cite{vovk1998game} or \cite[Section 3.6]{cesa-bianchi}. Its substitution $\Sigma_{L^{2}}^{[l,r]}$ is defined by
\begin{equation}\Sigma_{L^{2}}^{[l,r]}\big(\{\gamma^{n},{\alpha}^{n}\}_{n=1}^{N}\big)=\frac{r+l}{2}+\frac{(r-l)}{4}\log \frac{\sum_{n=1}^{N}{\alpha}^{n}\exp[-2\big(\frac{r-\gamma^{n}}{r-l}\big)^{2}]}{\sum_{n=1}^{N}{\alpha}^{n}\exp[-2\big(\frac{\gamma^{n}-l}{r-l}\big)^{2}]}.
\label{squared-loss-substitution}
\end{equation}
\label{square-loss-example}
\end{example}

\begin{example}[Logarithmic loss] The function $\lambda(\gamma,\omega)=-\log\gamma(\omega)$ with $\Omega=\{1,\dots,K\}$ and $\Gamma=\Delta_{K}$ is $1$-mixable, see \cite{adamskiy2012putting}. The substitution function $\Sigma_{log}$ is defined by
$$\big[\Sigma_{log}\big(\{\gamma^{n},{\alpha}^{n}\}_{n=1}^{N}\big)\big](k)=\sum_{n=1}^{N}{\alpha}^{n}\gamma^{n}(k).$$
\end{example}

\begin{definition}[Exponentially concave loss function] Let $\Gamma$ be a \textbf{convex subset of a linear space} over $\mathbb{R}$. A function $\lambda:\Gamma\times\Omega\rightarrow \mathbb{R}$ is called $\eta$-exponentially concave if for all $N=1,2,\dots$, probability vector $({\alpha}^{1},\dots,{\alpha}^{N})$ and vectors of forecast $(\gamma^{1},\dots,\gamma^{N})\in \Gamma^{N}$ the following holds true for all $\omega\in\Omega$:

$$\exp\big[-\eta\lambda(\overline{\gamma}, \omega)\big]\geq \sum_{n=1}^{N}{\alpha}^{n}\exp\big[-\eta \lambda(\gamma^{n},\omega)\big],$$
where $\overline{\gamma}=\sum_{n=1}^{N}{\alpha}^{n}\gamma^{n}$. Note that $\overline{\gamma}\in\Gamma$ due to the convexity of $\Gamma$.
\end{definition}
Clearly, \textbf{exponential concavity leads to mixability}.  {Indeed, one may naturally put
$\Sigma\big(\{\gamma^{n},{\alpha}^{n}\}_{n=1}^{N}\big)=\sum_{n=1}^{N}{\alpha}^{n}\gamma^{n}$ as the substitution function.} However, the inverse is not always true, see discussion in \cite{cesa-bianchi}.\footnote{In specific cases one may apply \textbf{exp-concavifying} transform to reparametrize the loss to make it exponentially concave, see \cite{kamalaruban2015exp}.}
Also exp-concavity is naturally defined only for convex subsets of linear spaces, while mixability can be defined on arbitrary sets.

The square loss function $\lambda(\gamma, \omega)=(\gamma-\omega)^{2}$ with ${\Gamma=\Omega=[l,r]}$ is $\frac{1}{2(r-l)^{2}}$-exponentially concave, see \cite{kivinen1999averaging}.\footnote{The maximal exponential concavity rate $\eta$ of squared loss is $4$ times lower than the corresponding mixability rate.} Also, from the previous subsection we see that the logarithmic loss is $1$-exponentially concave.

\section{Mixability (Exp-Concavity) of Integral Loss Functions}
\label{sec-int-mix}

In the framework of Protocol \ref{protocol} we consider function-valued forecasting. We prove that $\eta$-mixable (exp-concave) loss $\lambda$ for comparing single-value outcomes admits an $\eta$-mixable (exp-concave) integral extension for comparing function-valued outcomes.


\begin{theorem}[Mixability \& Exp-concavity of Integral Losses] Let $(\Gamma,\sigma_{\Gamma})$, $(\Omega,\sigma_{\Omega})$, $(\mathcal{X},\sigma_{\mathcal{X}})$ be measurable spaces.
Assume that ${\lambda:\Gamma\times\Omega\rightarrow\mathbb{R}_{+}}$ is a loss function measurable w.r.t. product $\sigma_{\Gamma}\times\sigma_{\Omega}$. Let $\lambda_{u,\mu}$ be $\mathcal{X}$-integral $\lambda$-loss function. Assume that the substitution function $\Sigma_{\lambda}$ is measurable. Then function $\lambda_{u,\mu}$ is $\eta$-mixable, and as a substitution function (for $N$ experts) we can use
$$\Sigma_{\lambda_{u,\mu}}:\big(\mathcal{M}(\Gamma^{\mathcal{X}})\big)^{N}\times \Delta_{N}\rightarrow \Gamma^{\mathcal{X}}$$
defined by point-wise ($x\in \mathcal{X}$) application of substitution function $\Sigma_{\lambda}$ for $\lambda$:
$$\Sigma_{\lambda_{u,\mu}}\big[\{\gamma^{n}, {\alpha}^{n}\}_{n=1}^{N}\big](x):=\Sigma_{\lambda}\big(\{\gamma^{n}(x), {\alpha}^{n}\}_{n=1}^{N}\big).$$
\label{theorem-main}
\end{theorem}
 We emphasize that the suggested substitution function $\Sigma_{\lambda_{u,\mu}}$ is \textbf{independent} of both $u$ and $\mu$. Thus, the same substitution function attains efficient prediction for all possible $\mathcal{X}$-integral $\lambda$-loss functions (all admissible $\mu$ and $u$), which may be even chosen by an adversary after the prediction at each step is made.

To prove our main Theorem \ref{theorem-main}, we will need the following
\begin{theorem}[Generalized Holder Inequality]
\label{generalized-holder}
Let $(\mathcal{X},\mu)$ and $(\mathcal{Y},\nu)$ denote two $\sigma$-finite measure spaces. Let $f(x,y)$ be positive and measurable on $(\mathcal{X}\times \mathcal{Y}, \mu\times \nu)$ function, and $u(x), v(y)$ be weight functions and ${\int_{\mathcal{X}}u(x)d\mu(x)=1}$. Then 
\begin{eqnarray}\int_{\mathcal{Y}}\exp\bigg(\int_{\mathcal{X}}\log f(x,y)u(x)d\mu(x)\bigg)v(y)d\nu(y)\leq 
\nonumber
\\
\exp\bigg(\int_{\mathcal{X}}\log \bigg[\int_{\mathcal{Y}}f(x,y)v(y)d\nu(y)\bigg]u(x)d\mu(x)\bigg).
\label{generalized-holder-inequality}
\end{eqnarray}
\label{theorem-generalized-holder-inequality}
\end{theorem}

\noindent An explicit discussion of inequality \eqref{generalized-holder-inequality} is provided in \cite{nikolova2017new}. Inequality \eqref{generalized-holder-inequality} is also known as \textbf{Continuous Form of Holder Inequality} by \cite{dunford1958linear} and \textbf{Extended Holder Inequality} by \cite{kwon1995extension}. The proof can be found within the mentioned works.
 
Now we prove our main Theorem \ref{theorem-main}. Consider the pool $\mathcal{N}$ of experts. Let $\gamma^{n}:\mathcal{X}\rightarrow \Gamma$ be their measurable  forecasts and 
 ${\alpha}^{1},\dots, {\alpha}^{N}$ be the experts' weights. We denote the forecast aggregated according to $\Sigma_{\lambda_{u,\mu}}$ by $\overline{\gamma}\in\mathcal{M}(\Gamma^{\mathcal{X}})$.
 
 In the following proof we will directly check that for every $\omega\in\mathcal{M}(\Omega^{\mathcal{X}})$ it holds true (in fact, for all admissible $u$ and $\mu$) that
 \begin{equation}
 \exp\big[-\eta \lambda_{u,\mu}(\overline{\gamma},\omega)\big]\geq \sum_{n=1}^{N}{\alpha}^{n}\exp(-\eta \lambda_{u,\mu}(\gamma^{n},\omega)),
 \label{integral-mixability-requirement}
 \end{equation}
 so that $\Sigma_{\lambda_{u,\mu}}$ is a proper substitution function and $\lambda_{u,\mu}$ is indeed $\eta$-mixable.

\begin{proof}

Choose any $\omega\in \mathcal{M}(\Omega^{\mathcal{X}})$. Since $\lambda$ is $\eta$-mixable with substitution function $\Sigma_{\lambda}$, for all $x\in \mathcal{X}$ we have
$$\exp\big[-\eta\lambda\big(\overline{\gamma}(x), \omega(x)\big)\big]\geq 
\sum_{n=1}^{N}{\alpha}^{n}\bigg[\exp\big[-\eta \lambda\big(
\gamma^{n}(x), \omega(x)\big)\big]\bigg]
.$$
We take the logarithm of both parts of the inequality and  for every $x\in \mathcal{X}$ obtain
$$
-\eta\lambda\big(\overline{\gamma}(x), \omega(x)\big)\geq \log\sum_{n=1}^{N}{\alpha}^{n}
\bigg[\exp \big[-\eta \lambda\big(\gamma^{n}(x), \omega(x)\big)\big]\bigg].
$$
We multiply both sides by $u_{\omega}(x)\geq 0$ and integrate over all $x\in \mathcal{X}$ w.r.t. measure $\mu_{\omega}$:
\begin{eqnarray}\underbrace{\int_{\mathcal{X}}\big[-\eta \lambda\big(\overline{\gamma}(x), \omega(x)\big)\big]u_{\omega}(x)d\mu_{\omega}(x)}_{-\eta \lambda_{u,\mu}(\overline{\gamma},\omega)}\geq
\nonumber
\\
\int_{\mathcal{X}}\log\sum_{n=1}^{N}{\alpha}^{n}\bigg[\exp \big[-\eta \lambda\big(\gamma^{n}(x), \omega(x)\big)\big]\bigg]u_{\omega}(x)d\mu_{\omega}(x)
\label{intemediate-mix}
\end{eqnarray}
The left part of inequality \eqref{intemediate-mix} equals to $-\eta \lambda_{u,\mu}(\overline{\gamma},\omega)$. Next, for $x\in\mathcal{X}$ and $n\in\mathcal{N}$ we define $$f(x,n):=\exp \big[-\eta \lambda\big(\gamma^{n}(x), \omega(x)\big)\big].$$
By applying the notation change and taking the exponent of both sides of \eqref{intemediate-mix}, we obtain

\begin{equation}
\exp\big[-\eta \lambda_{u,\mu}(\overline{\gamma},\omega)\big]\geq \exp\bigg(\int_{\mathcal{X}}\log \big[\sum_{n=1}^{N}{\alpha}^{n}\cdot f(x,n)\big] u_{\omega}(x)d\mu_{\omega}(x)\bigg)
\label{exp-intemediate-mix}
\end{equation}
The final step is to apply Theorem \ref{theorem-generalized-holder-inequality}, i.e. Generalized Holder inequality \eqref{generalized-holder}. In the notation of Theorem we use $\mathcal{Y}:=\mathcal{N}$, $v(y)\equiv 1$, $\nu(y):={\alpha}^{y}$ and obtain
\begin{eqnarray}\exp\bigg(\int_{\mathcal{X}}\log \big[\sum_{n=1}^{N}{\alpha}^{n}\cdot f(x,n)\big] u_{\omega}(x)d\mu_{\omega}(x)\bigg)\geq
\nonumber
\\
\sum_{n=1}^{N}{\alpha}^{n}\exp\bigg(\int_{\mathcal{X}}\log\big(f(x,n)\big)u_{\omega}(x)d\mu(x)\bigg)=
\nonumber
\\
\sum_{n=1}^{N}{\alpha}^{n}\bigg[\exp\bigg(\int_{\mathcal{X}}\underbrace{\log\big(f(x,n)\big)}_{-\eta \lambda(\gamma^{n}(x),\omega(x))}u_{\omega}(x)d\mu_{\omega}(x)\bigg)\bigg]=
\nonumber
\\
\sum_{n=1}^{N}{\alpha}^{n}\bigg[\exp\bigg(-\eta\underbrace{\int_{\mathcal{X}}\lambda(\gamma^{n}(x),\omega(x))u(x)dx}_{\lambda_{u,\mu}(\gamma^{n},\omega)}\bigg)\bigg]=
\nonumber
\\
\sum_{n=1}^{N}{\alpha}^{n}\exp(-\eta \lambda_{u,\mu}(\gamma^{n},\omega))
\label{holder-to-mix}
\end{eqnarray}
Now we combine \eqref{holder-to-mix} with \eqref{exp-intemediate-mix} and obtain the desired inequality \eqref{integral-mixability-requirement}.
\end{proof}




\section{Forecasting of Probability Distributions}
\label{sec-forecasting-probs}
In this section, we consider online probabilistic forecasting. At each step of the game experts provide forecasts as probability distributions on $\mathcal{X}$. The learner has to aggregate forecasts into a single forecast being a probability distribution on $\mathcal{X}$. Next, the true probability distribution, probably empirical, is revealed. Both the experts and the learner suffer losses using a loss function.

We analyse loss functions which are widely used for comparing probability distributions, show that they are actually integral loss functions and prove their mixability (exp-concavity). The results are summarised in Table 1 of Subsection \ref{sec-table}. Each following subsection is devoted to a particular loss function.

\subsection{Mixability \& Exp-concavity Table}
\label{sec-table}

Our results on mixability (exp-concavity) of common loss functions used to compare probability distributions are summarised in Table 1. The column ``\textbf{empirical outcomes}'' indicates whether it is possible to use the loss function to compare the predicted distribution $\gamma$ (typically continuous) with a discrete outcome $\omega$ (empirical distribution).

\begin{table}[h]
\tiny
\centering
\begin{tabular}{|c|c|c|c|c|}
\hline
\multirow{2}{*}{\textbf{Loss Function}}                                                         & \multirow{2}{*}{\textbf{Specifications}}                                                                                                                                                                                                                & \multirow{2}{*}{\textbf{\begin{tabular}[c]{@{}c@{}}Empirical\\ Outcomes\end{tabular}}} & \multicolumn{2}{c|}{\textbf{Aggregating Rule}}                                                                                                                                                                          \\ \cline{4-5} 
                                                                                                &                                                                                                                                                                                                                                                         &                                                                                        & \textbf{Exp-concavity}                                                                                        & \textbf{Mixability}                                                                                     \\ \hline
\begin{tabular}[c]{@{}c@{}}Continuous Ranked\\ Probability Score\\ (CRPS)\end{tabular}         & \begin{tabular}[c]{@{}c@{}}Borel distributions on\\ $\mathcal{X}=\prod_{d=1}^{D}[a_{d},b_{d}]$\end{tabular}                                                                                                                                                   & Possible                                                                                    & \begin{tabular}[c]{@{}c@{}}Mixture for\\ $\eta=\frac{1}{2\prod_{d=1}^{D}(b_{d}-a_{d})}$\end{tabular}          & \begin{tabular}[c]{@{}c@{}}Formula for\\ $\eta=\frac{2}{\prod_{d=1}^{D}(b_{d}-a_{d})}$\end{tabular}          \\ \hline
\begin{tabular}[c]{@{}c@{}}Sliced Continuous\\ Ranked Probability\\ Score (SCRPS)\end{tabular} & \begin{tabular}[c]{@{}c@{}}Borel distributions on\\ $\mathcal{X}\subset \text{Ball}_{\mathbb{R}^{D}}(0, R)$\end{tabular}                                                                                                                                      & Possible                                                                                    & \begin{tabular}[c]{@{}c@{}}Mixture for\\ $\eta=\frac{1}{8R}$\end{tabular}                              & \begin{tabular}[c]{@{}c@{}}No closed form,\\ $\eta=\frac{1}{2R}$\end{tabular}                        \\ \hline
\begin{tabular}[c]{@{}c@{}}Energy-Based\\ Distance ($\mathcal{E}$)\end{tabular} & \begin{tabular}[c]{@{}c@{}}Borel distributions on\\ $\mathcal{X}\subset \text{Ball}_{\mathbb{R}^{D}}(0, R)$\end{tabular}                                                                                                                                      & Possible                                                                                    & \begin{tabular}[c]{@{}c@{}}Mixture for\\ $\eta=\frac{S_{D-2}}{8R(D-1)S_{D-1}}$\end{tabular}                              & \begin{tabular}[c]{@{}c@{}}No closed form,\\ $\eta=\frac{S_{D-2}}{2R(D-1)S_{D-1}}$\end{tabular}                        \\ \hline
\begin{tabular}[c]{@{}c@{}}Kullback–Leibler\\ divergence (KL)\end{tabular}                      & \begin{tabular}[c]{@{}c@{}}Distributions on\\ $(\mathcal{X},\sigma,\mu)$ with\\ non-zero density\end{tabular}                                                                                                                                           & \begin{tabular}[c]{@{}c@{}}Reduces to\\ Log-loss\end{tabular}                                                                                & \multicolumn{2}{c|}{Mixture for $\eta=1$}                                                                                                                                                                               \\ \hline
\begin{tabular}[c]{@{}c@{}}Beta-2 divergence\\ ($\mathcal{B}_{2}$)\end{tabular}                                                                               & \begin{tabular}[c]{@{}c@{}}Distributions on\\ $(\mathcal{X},\sigma,\mu)$ with $\|\mu\|_{1}<\infty$\\ and $M$-bounded density\end{tabular}                                                                                                                  & If $|\mathcal{X}|<\infty$                                                                                     & \begin{tabular}[c]{@{}c@{}}Mixture for\\ $\eta=\frac{1}{2\|\mu\|_{1}M^{2}}$\end{tabular}                            & \begin{tabular}[c]{@{}c@{}}No closed form,\\ $\eta=\frac{2}{\|\mu\|_{1}M^{2}}$\end{tabular}                 \\ \hline
\begin{tabular}[c]{@{}c@{}}Characteristic\\ function\\ discrepancy (CFD)\end{tabular}            & \begin{tabular}[c]{@{}c@{}}Borel distributions\\ on $\mathcal{X}\subset\mathbb{R}^{D}$\end{tabular}                                                                                                                                                           & Possible                                                                                    & \begin{tabular}[c]{@{}c@{}}Mixture for\\ $\eta=\frac{1}{8}$\end{tabular}                                      & \begin{tabular}[c]{@{}c@{}}No closed form,\\ $\eta=\frac{1}{4}$\end{tabular}                          \\ \hline
\begin{tabular}[c]{@{}c@{}}Maximum mean\\ discrepancy (MMD)\end{tabular}            & \begin{tabular}[c]{@{}c@{}}Borel distributions\\ on $\mathcal{X}\subset\mathbb{R}^{D}$,\\Positive definite kernel\\$k(x,y)=\psi(x-y)$ for\\positive definite function\\$\psi(x)=\int_{\mathbb{R}^{D}}e^{-i\langle x,t\rangle}d\mu(t)$\end{tabular}                                                                                                                                                           & Possible                                                                                    & \begin{tabular}[c]{@{}c@{}}Mixture for\\ $\eta=\frac{1}{8\|\mu\|_{1}}$\end{tabular}                                      & \begin{tabular}[c]{@{}c@{}}No closed form,\\ $\eta=\frac{1}{4\|\mu\|_{1}}$\end{tabular}                          \\ \hline
\begin{tabular}[c]{@{}c@{}}1-dimensional\\ optimal transport\\ cost (OT)\end{tabular}           & \begin{tabular}[c]{@{}c@{}}Borel distributions\\ on $\mathcal{X}\subset\mathbb{R}$;\\ $\eta$-mixable\\ (exp-concave)\\ cost $c:\mathcal{X}\times\mathcal{X}\rightarrow \mathbb{R}$\\ satistying $\frac{\partial^{2} c}{\partial x\partial x'}<0$\end{tabular} & Possible                                                                                    & \begin{tabular}[c]{@{}c@{}}Wasserstein-2\\ barycenter\end{tabular}                                            & \begin{tabular}[c]{@{}c@{}}Mixed quantile\\for monotone\\ substitution $\Sigma_{c}$;\\No closed form\\for arbitrary\\substitution,\\but can be modeled\\implicitly (Lemma \ref{lemma-implicit})\end{tabular} \\ \hline
\begin{tabular}[c]{@{}c@{}}Sliced \\ Wasserstein-2\\ distance ($\text{SW}_{2}$)\end{tabular}              & \begin{tabular}[c]{@{}c@{}}Radon distributions on\\ $\mathcal{X}\subset \text{Ball}_{\mathbb{R}^{D}}(0, R)$,\\scaled and translated\\copies of each other.\end{tabular}                                                                                                                    & Possible                                                                                    & \begin{tabular}[c]{@{}c@{}}Sliced\\ Wasserstein-2\\ barycenter for\\ $\eta=\frac{1}{8R^{2}}$\\(Lemma \ref{lemma-scale-rot-bar})\end{tabular} & \begin{tabular}[c]{@{}c@{}}Unknown,\\ $\eta=\frac{1}{2R^{2}}$?\\(see Subsection \ref{sec-sliced-ot})\end{tabular}              \\ \hline
\end{tabular}
\label{mix-table}
\caption{Mixability \& exp-concavity of various loss functions used for assessing probabilitic forecasts.}
\end{table}

\subsection{Continuous Ranked Probability Score}
\label{sec-crps}

\subsubsection{One-Dimensional Case}

Let $\mathcal{X}=[a,b]\subset \mathbb{R}$ and assume that $\Gamma=\Omega$ is the space of all probability measures over Borel field of $\mathcal{X}$. In this case \textbf{Continuous Ranked Probability Score} (CRPS) by \cite{matheson1976scoring} is widely used for comparing probability distributions:
\begin{equation}
\text{CRPS}(\gamma,\omega)=\int_{a}^{b}|\text{CDF}_{\gamma}(x)-\text{CDF}_{\omega}(x)|^{2}dx,
\label{crps-definition}
\end{equation}
where by $\text{CDF}_{\upsilon}:\mathcal{X}\rightarrow[0,1]$ we denote cumulative distribution function of probability distribution $\upsilon\in\Gamma$, i.e. ${\text{CDF}_{\upsilon}(x)=\upsilon([a,x])}$ for all $x\in[a,b]$. We visualize CRPS in the following Figure \ref{fig:crps}.

\begin{figure}[!h]
     \centering
     \begin{subfigure}[b]{0.48\columnwidth}
         \centering
         \includegraphics[width=\linewidth]{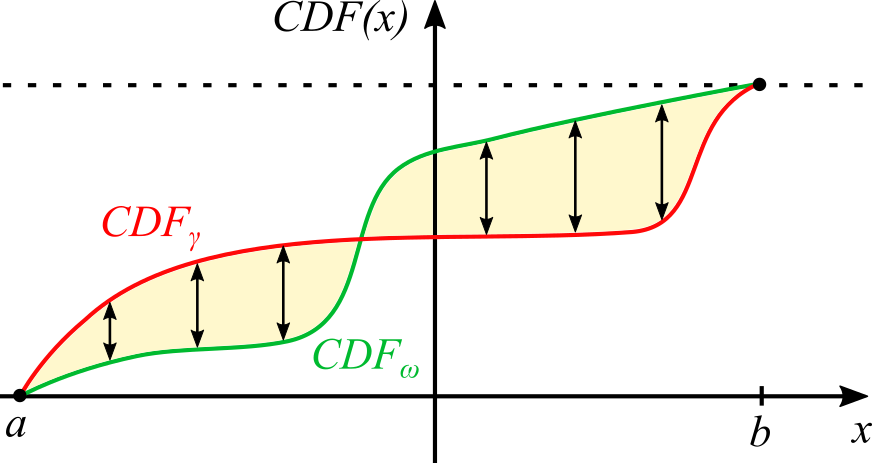}
        \caption{Arbitrary distribution $\omega$.}
     \end{subfigure}
    \begin{subfigure}[b]{0.48\columnwidth}
        \centering
    \includegraphics[width=\linewidth]{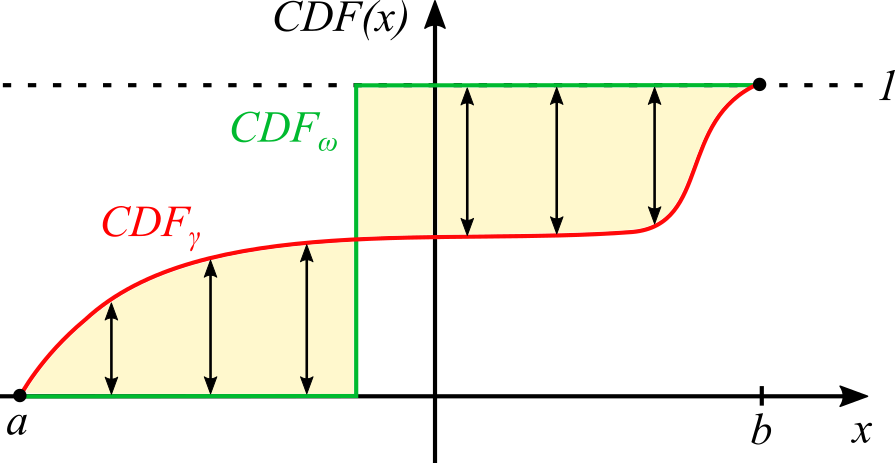}
    \caption{Empirical (Dirac) distribution $\omega$.}
     \end{subfigure}
    \caption{Visualization of the comparison of CDFs of distributions $\gamma,\omega$ on $[a,b]$ by using Continuous Ranked Probability Score.}
    \label{fig:crps}
\end{figure}

In practice, CRPS is used, for example, for assessing the quality of forecasts of different weather-related variables \cite{zamo2018estimation}.
As we know, so far CRPS is the only non-trivial loss on continuous outcomes which is already known to be mixable. Indeed, \cite{v2019online,dzhamtyrova2019competitive} proved the mixability of CRPS but only for the case when the outcome $\omega$ is a Dirac distribution $\omega=\delta_{x_{0}}$ for some $x_{0}\in[a,b]$.

Within our framework CRPS equals to the integral square loss on $[a,b]$ with density $u_{\omega}(x)\equiv \frac{1}{b-a}$ and Lebesgue measure $\mu_{\omega}$, multiplied by $(b-a)$:
\begin{equation}\text{CRPS}(\gamma,\omega)=(b-a)\underbrace{\int_{a}^{b}|\text{CDF}_{\gamma}(x)-\text{CDF}_{\omega}(x)|^{2}\overbrace{\frac{1}{b-a}}^{u_{\omega}(x)}dx}_{\text{Integral squared loss}}.
\label{crps-to-integral}
\end{equation}
{For all $x\in[a,b]$ it holds true that $\text{CDF}_{\gamma}(x), \text{CDF}_{\omega}(x)\in [0, 1]$. The function ${|\text{CDF}_{\gamma}(x)-\text{CDF}_{\omega}(x)|^{2}}$ is a square loss function ${[0,1]\times[0,1]\rightarrow \mathbb{R}}$ of inputs $\text{CDF}_{\gamma},\text{CDF}_{\omega}$. It is $2$-mixable ($\frac{1}{2}$-exp-concave), see Example \ref{square-loss-example}.} Thus, by Theorem \ref{theorem-main} the corresponding integral loss on $[a,b]$ is also $2$-mixable ($\frac{1}{2}$-exp-concave). We conclude that CRPS is mixable (exp-concave) but with $(b-a)$ times lower learning rate, i.e. $\eta=\frac{2}{(b-a)}$ (or $\eta=\frac{1}{2(b-a)}$ for exp-concavity). 

For $\frac{2}{(b-a)}$-mixable CRPS the aggregated prediction $\overline{\gamma}$ is given by its CDF:
\begin{equation}\overline{\text{CDF}}(x)=\frac{1}{2}+\frac{1}{4}\log \frac{\sum_{n=1}^{N}{\alpha}^{n}\exp[-2\big(1-\text{CDF}_{\gamma^{n}}(x)\big)^{2}]}{\sum_{n=1}^{N}{\alpha}^{n}\exp[-2\big(\text{CDF}_{\gamma^{n}}(x)\big)^{2}]},
\label{crps-substitution}
\end{equation}
for weights ${({\alpha}^{1},\dots,{\alpha}^{N})\in\Delta_{N}}$ and experts' predictions $\gamma^{1},\dots,\gamma^{N}\in\Gamma$. It equals to the point-wise application of $\Sigma_{L^{2}}^{[0,1]}$ (defined by \eqref{squared-loss-substitution}) to experts' forecasts CDF's. One may perform straightforward check to verify that the resulting aggregated function \eqref{crps-substitution} is indeed a CDF. 

For $\frac{1}{2(b-a)}$-exponentially concave CRPS we have
\begin{equation}
\overline{\text{CDF}}(x)={\sum_{n=1}^{N}{\alpha}^{n}\cdot \text{CDF}_{\gamma_{n}}(x)},
\label{crps-wa-substitution}
\end{equation}
which means that aggregated prediction $\overline{\gamma}$ is a mixture of probability distributions $\gamma^{n}$ w.r.t. weights ${\alpha}^{n}$.

\subsubsection{Multi-Dimensional Case}
\label{sec-scrps}
Let $\mathcal{X}\subset \mathbb{R}^{D}$ be a compact subset. Assume that $\Gamma=\Omega$ is the space of all probability measures over Borel field of $\mathcal{X}$. One may naturally extend CRPS formula \eqref{crps-definition} for comparing multi-dimensional distributions $\gamma,\omega$ on $\mathcal{X}$:
\begin{equation}\text{CRPS}(\text{CDF}_{\gamma},\text{CDF}_{\omega})=\int_{\mathcal{X}}|\text{CDF}_{\gamma}(x)-\text{CDF}_{\omega}(x)|^{2}dx,
\label{crps-multidim-definition}
\end{equation}
where $\text{CDF}_{\upsilon}:\mathcal{X}\rightarrow [0,1]$ denotes CDF of multi-dimensional $\upsilon\in\Gamma$. Similar to the analysis of the previous subsection, one may prove that if $\mathcal{X}\subset \prod_{d=1}^{D}[a_{d},b_{d}]$, then the loss \eqref{crps-multidim-definition} is ${\big[2\cdot \big(\prod_{d=1}^{D}(b_{d}-a_{d})\big)^{-1}\big]}$-mixable (and $4$-times lower exp-concave), and obtain an analogue to aggregation rule \eqref{crps-substitution} for mixability and \eqref{crps-wa-substitution} for exp-concavity. The analogue to \eqref{crps-wa-substitution} is straightforward: the aggregated distribution is a mixture. However, the substitution analogue to \eqref{crps-substitution} is hard to use in high-dimensional spaces, e.g. to sample from the distribution or to calculate moments.
Besides, definition \eqref{crps-multidim-definition} is not symmetric to probability measure rotations and is sensitive to the choice of the coordinate system.

To overcome above-mentioned issues, we propose to use \textbf{Sliced Continuous Ranked Probability Score} (SCRPS) which reduces to estimation of multiple one-dimensional CRPS and is invariant to the choice of the coordinate system. First, for all $\upsilon\in\Gamma$ we define $\text{SCDF}: \mathbb{S}^{D-1}\times \mathbb{R}\rightarrow [0,1]$ by
$$\text{SCDF}_{\upsilon}(\theta,t)=\upsilon\big(\{x\in\mathcal{X}\text{ }|\text{ }\langle x,\theta\rangle\leq t\}\big),$$
where $\theta\in \mathbb{S}^{D-1}=\{x\in \mathbb{R}^{D}\wedge \|x\|_{2}=1\}$ is a unit sphere. Function $\text{SCDF}_{\upsilon}(\theta,\cdot)$ is a CDF of the distribution $\upsilon$ projected on to the line orthogonal to $\langle x,\theta\rangle=0$. Let $S_{D-1}$ be the surface area of $(D-1)$-dimensional unit sphere. Now we define SCRPS:
\begin{equation}\text{SCRPS}(\gamma,\omega)=\frac{1}{S_{D-1}}\int_{\mathbb{S}^{D-1}}\bigg[\int_{-\infty}^{\infty}|\text{SCDF}_{\gamma}(\theta,t)-\text{SCDF}_{\omega}(\theta,t)|^{2}dt\bigg]d\theta.
\label{scrps-definition}
\end{equation}
From the definition we see that SCRPS is the average over all the sliced CRPS scores. Thus, similar to \textbf{Sliced Wasserstein distances} \cite{kolouri2018sliced},
its stochastic computation in practice can be efficiently preformed via projections on random directions ${\theta\in \mathbb{S}^{D-1}}$.

Let us prove that SCRPS is mixable (exp-concave) for bounded $\mathcal{X}\subset \text{Ball}_{\mathbb{R}^{D}}(0,R)$. In this case limits $\pm \infty$ of the inner integral in \eqref{scrps-definition} are replaced by $\pm R$ respectively. We have

\begin{eqnarray}\text{SCRPS}(\gamma,\omega)=
\nonumber
\\
2R\underbrace{\int_{\mathbb{S}^{D-1}\times [-R,R]}|\text{SCDF}_{\gamma}(\theta,t)-\text{SCDF}_{\omega}(\theta,t)|^{2}\frac{1}{2RS_{D-1}}d\theta dt}_{\text{Integral squared loss}},
\label{scrps-to-integral}
\end{eqnarray}
which similar to \eqref{crps-to-integral} reduces SCRPS to a multiple integral of the squared loss, which is $2$-mixable ($\frac{1}{2}$-exp-concave) on $[0,1]$. Thus, the integral squared loss is $2$-mixable ($\frac{1}{2}$-exp-concave) by Theorem \ref{theorem-main}. We conclude that SCRPS is also mixable (exp-concave) but with $2R$ lower learning rate.

For $\frac{1}{2R}$-exp-concave SCRPS the aggregated prediction $\overline{\gamma}$ is given by its sliced CDF:
\begin{equation}\overline{\text{SCDF}}(\theta,t)=\sum_{n=1}^{N}{\alpha}^n \cdot\text{SCDF}_{\gamma^{n}}(\theta,t)
\label{scrps-exp-substitution}\end{equation}
for experts' predictions $\gamma^{1},\dots,\gamma^{N}\in\Gamma$ and weights ${({\alpha}^{1},\dots,{\alpha}^{N})\in\Delta_{N}}$. The mixture $\overline{\gamma}$ of probability distributions $\gamma^{n}$ w.r.t. weights ${\alpha}^{n}$ can be used as an aggregated prediction because its sliced CDF equals \eqref{scrps-exp-substitution}.

For $\frac{2}{R}$-mixable SCRPS the aggregated prediction's sliced CDF should satisfy
\begin{equation}\overline{\text{SCDF}}(\theta,t)=\frac{1}{2}+\frac{1}{4}\log \frac{\sum_{n=1}^{N}{\alpha}^{n}\exp[-2\big(1-\text{SCDF}_{\gamma^{n}}(\theta,t)\big)^{2}]}{\sum_{n=1}^{N}{\alpha}^{n}\exp[-2\big(\text{SCDF}_{\gamma^{n}}(\theta,t)\big)^{2}]}.
\label{scrps-substitution}
\end{equation}
However, we do not know whether $\overline{\text{SCDF}}(\theta,t)$ is necessarily a \textbf{sliced CDF} of some distribution $\overline{\gamma}$. To fix this issue, one may consider the \textbf{projection trick} and define
$$\overline{\gamma}=\argmin_{\gamma\in\Gamma}\bigg(\frac{1}{S_{D-1}}\int_{\mathbb{S}^{D-1}}\bigg[\int_{-\infty}^{\infty}|\text{SCDF}_{\gamma}(\theta,t)-\overline{\text{SCDF}}(\theta,t)|^{2}dt\bigg]d\theta\bigg).$$
SCRPS is a squared norm of difference of SCDFs in ${\mathcal{L}^{2}:=\mathcal{L}^{2}(\mathbb{S}^{D-1}\times [-R,R])}$ of quadratically integrable functions (w.r.t. the product of uniform probability measure on $\mathbb{S}^{D-1}$ and Lebesgue measure on $[-R,R]$). Thus, $\text{SCDF}_{\overline{\gamma}}$ can be viewed as the projection of $\overline{\text{SCDF}}$ on to the convex subset of sliced cumulative distribution functions. Hence, for all $\omega\in\Omega$ it satisfies  $$\|\text{SCDF}_{\omega}-\overline{\text{SCDF}}\|_{\mathcal{L}^{2}}^{2}\geq \|\text{SCDF}_{\omega}-\text{SCDF}_{\overline{\gamma}}\|_{\mathcal{L}^{2}}^{2}=\text{SCRPS}(\omega,\overline{\gamma}),$$
i.e. distribution $\overline{\gamma}$ can be used as the aggregated prediction.

Although the projection trick potentially allows to obtain admissible aggregated SCDF and the corresponding distribution, we do not know whether in a general practical case the computation of $\overline{\gamma}$ is feasible for SCRPS. Thus, from the computational point of view, formula \eqref{scrps-exp-substitution} should be preferred despite the fact it provides $4$-times lower learning rate.



\subsection{Energy-Based Distance}

Let $\mathcal{X}\subset \mathbb{R}^{D}$ be a compact subset and $\Gamma=\Omega$ is the space of all probability measures over Borel field of $\mathcal{X}$. For $\gamma,\omega\in\Gamma$ we consider \textbf{Energy-Based Distance}:
\begin{equation}
    \mathcal{E}(\gamma,\omega)=2\cdot \mathbb{E}_{\gamma\times\omega}\|x-y\|_{2}-\mathbb{E}_{\gamma\times\gamma}\|x-x'\|_{2}-\mathbb{E}_{\omega\times\omega}\|y-y'\|_{2}.
    \label{energy-based-distance}
\end{equation}
It is absolutely non trivial to show that $\mathcal{E}$ is actually defines a metric on the space of probability distributions, see e.g. \cite{rizzo2016energy} for explanations. Formula \eqref{energy-based-distance}
naturally admits unbiased estimates from empirical samples, which makes Energy-Based distance attractive for the usage in generative models, see \cite{bellemare2017cramer}. 


Note that at the first glance Energy-based distance defines $\mathcal{L}^{1}$-like metric because it operates with distances in $\mathcal{X}\subset \mathbb{R}^{D}$ rather than their squares. Since $\mathcal{L}^{1}$-loss is not mixable (unlike squared loss $\mathcal{L}^{2}$), it is reasonable to expect that $\mathcal{E}$ is also not mixable (exp-concave). 

Surprisingly, Energy-based distance equals SCRPS up to a constant which depends on the dimension $D$. Thus, Energy-based distance is mixable (exp-concave). Equivalence is known for $D=1$, see \cite{rizzo2016energy}.
We prove the equivalence for arbitrary ${D\geq 1}$. 
\begin{theorem}[Equivalence of SCRPS and Energy-Based Distance]Let $\gamma,\omega$ be two Borel probability measures on $\mathcal{X}=\mathbb{R}^{D}$ with finite first moments, i.e. $\mathbb{E}_{X\sim\gamma}|X|, \mathbb{E}_{Y\sim\gamma}|Y|<\infty$. Then for $D>1$
$$\mathcal{E}(\gamma,\omega)=(D-1)\frac{S_{D-1}}{S_{D-2}}\text{\normalfont SCRPS}(\gamma,\omega).$$
\label{thm-energy-scrps}
\end{theorem}
\begin{proof}
For $x\in\mathcal{X}$ consider the value $s(x)=\int_{\theta\in \mathbb{S}^{D-1}}|\langle x,\theta\rangle| d\theta$. Note that $s(x)$ depends only on $\|x\|$, i.e. 
$$s(x)=\int_{\theta\in \mathbb{S}^{D-1}}|\langle x,\theta\rangle| d\theta=\|x\|\cdot \int_{\theta\in \mathbb{S}^{D-1}}|\theta_{1}| d\theta=\|x\|\cdot s(1),$$ 
where $\theta_{1}$ is the first coordinate of $\theta=(\theta_{1},\dots,\theta_{D})\in\mathbb{S}^{D-1}\subset\mathbb{R}^{D}$. Thus,
\begin{eqnarray}\mathbb{E}_{\gamma\times\gamma}\|x-x'\|=\mathbb{E}_{\gamma\times\gamma}\frac{s(x-x')}{s(1)}=\nonumber
\\
\frac{1}{s(1)}\mathbb{E}_{\gamma\times\gamma}\big[\int_{\theta\in \mathbb{S}^{D-1}}|\langle x-x',\theta\rangle| d\theta\big]=
\nonumber
\\
\frac{1}{s(1)}\int_{\theta\in \mathbb{S}^{D-1}}\big[\mathbb{E}_{\gamma\times\gamma}|\langle x-x',\theta\rangle|\big] d\theta=
\nonumber
\\
\frac{1}{s(1)}\int_{\theta\in \mathbb{S}^{D-1}}\big[\mathbb{E}_{\gamma\times\gamma}|\langle x,\theta\rangle-\langle x',\theta\rangle|\big] d\theta
\label{abs-to-slice}
\end{eqnarray}
We derive analogs of \eqref{abs-to-slice} for other terms of \eqref{energy-based-distance} and obtain
\begin{eqnarray}\mathcal{E}(\gamma,\omega)=\frac{1}{s(1)}\int_{\theta\in \mathbb{S}^{D-1}}\bigg[2\mathbb{E}_{\gamma\times\omega}|\langle x,\theta\rangle-\langle y,\theta\rangle|
\nonumber
\\
-\mathbb{E}_{\gamma\times\gamma}|\langle x,\theta\rangle-\langle x',\theta\rangle|-\mathbb{E}_{\omega\times\omega}|\langle y,\theta\rangle-\langle y',\theta\rangle|\bigg] d\theta
\label{energy-to-scrps-int}
\end{eqnarray}
The expression within the large square brackets of \eqref{energy-to-scrps-int} equals Energy-based score between 1-dimensional projections of $\gamma$ and $\omega$ onto the direction $\theta$. According to \cite{rizzo2016energy},
it equals to CRPS multiplied by $2$, i.e. \eqref{energy-to-scrps-int} turns into
\begin{eqnarray}\mathcal{E}(\gamma,\omega)=\frac{1}{s(1)}\int_{\theta\in \mathbb{S}^{D-1}}\bigg[2\int_{-\infty}^{\infty}|\text{SCDF}_{\gamma}(\theta,t)-\text{SCDF}_{\omega}(\theta,t)|^{2}dt\bigg] d\theta=
\nonumber
\\
\frac{2S_{D-1}}{s(1)}\cdot \text{SCRPS}(\gamma,\omega)
\nonumber
\end{eqnarray}
Now we compute $s(1)$. Let $\mathbb{B}^{D}=\text{Ball}_{\mathbb{R}^{D}}(0,1)$ be the $D$-dimensional unit ball (whose boundary is $\mathbb{S}^{D-1}$). Let $V^{D}$ be the volume of $\mathbb{B}^{D}$. We note that
\begin{eqnarray}
\int_{\mathbb{B}^{D}}|\theta_{1}|d\theta=\int_{-1}^{1}|\theta_{1}|\cdot \big[(\sqrt{1-\theta_{1}^{2}})^{D-1}\cdot V^{D-1}\big]d\theta_{1}=
\nonumber
\\
2\int_{0}^{1}\theta_{1}\cdot \big[(\sqrt{1-\theta_{1}^{2}})^{D-1}\cdot V^{D-1}\big]d\theta_{1}=\frac{2V^{D-1}}{D+1},
\label{int-using-slices}
\end{eqnarray}
where we to compute the integral we decompose it into ball sliced orthogonal to the first axis.
Now we compute the integral again, but by decomposing it into the integrals over spheres:
\begin{eqnarray}
\int_{\mathbb{B}^{D}}|\theta_{1}|d\theta=\int_{0}^{1}\big[\int_{r\cdot\mathbb{S}^{D-1}}|\theta_{1}|d\theta\big]dr=\int_{0}^{1}r^{D}\underbrace{\big[\int_{\mathbb{S}^{D-1}}|\theta_{1}|d\theta\big]}_{s(1)}dr=\frac{s(1)}{D+1}.
\label{int-using-spheres}
\end{eqnarray}
Finally, by matching \eqref{int-using-slices} with \eqref{int-using-spheres} and using equality $V^{D-1}=\frac{S^{D-2}}{D-1}$ we conclude that $s(1)=2V^{D-1}=\frac{2S^{D-2}}{D-1}$ and finish the proof.
\end{proof}

From Theorem \ref{thm-energy-scrps} we immediately conclude that Energy-Based loss is mixable (exp-concave). If $\mathcal{X}\subset\text{Ball}_{\mathbb{R}^{D}}(0,R)$, then the learning rate for mixability (exp-concavity) is $(D-1)\frac{S_{D-1}}{S_{D-2}}$ times lower than the analogous rate for SCRPS. The aggregated prediction is computed exactly the same way as for SCRPS, see discussion in previous Subsection \ref{sec-scrps}.

\subsection{Density-Based Losses}
\label{sec-density-based}
Let $(\mathcal{X},\sigma,\mu)$ be a probability space with $\sigma$-finite measure $\mu$. Denote the set of all absolutely continuous w.r.t. $\mu$ probability measures on $(\mathcal{X},\sigma)$ by $\Omega=\Gamma$. For every $\upsilon\in\Gamma$ we denote its density by $p_{\upsilon}(x)=\frac{d\upsilon(x)}{d\mu(x)}$.

\subsubsection{Kullback-Leibler Divergence}
\label{sec-rkl}
Consider \textbf{Kullback-Leibler Divergence} between the outcome and the predicted distribution:
\begin{eqnarray}\text{KL}(\omega||\gamma)=-\int_{\mathcal{X}}\log\bigg[\frac{p_{\gamma}(x)}{p_{\omega}(x)}\bigg]p_{\omega}(x)d\mu(x)=
\nonumber
\\
\underbrace{-\int_{\mathcal{X}}\log [p_{\gamma}(x)] \cdot p_{\omega}(x)d\mu(x)}_{-\int_{\mathcal{X}}\log [p_{\gamma}(x)] d\omega(x)}-H_{\mu}(p_{\omega}).
\label{kl-divergence}
\end{eqnarray}
KL is a key tool in Bayesian machine learning. It is probably the most known representative of the class of $f$-divergences \cite{nowozin2016f}.

By skipping the $\gamma$-independent (prediction) entropy term $-H_{\mu}(p_{\omega})$ one may clearly see that the resulting loss is integral loss with density $u_{\omega}(x)=p_{\omega}(x)$ and $\mu_{\omega}\equiv \mu$ for the logarithmic loss function. The logarithmic function is $1$-mixable (exp-concave), thus, from Theorem \ref{theorem-main} we conclude that KL divergence is also $1$-mixable (exp-concave).

For $1$-mixable (exp-concave) KL-divergence the aggregated prediction  $\overline{\gamma}$ is given by its density
$$\overline{p}(x)=\sum_{n=1}^{N}{\alpha}^{n}p_{\gamma^{n}}(x)$$
for experts' predictions $\gamma^{1},\dots,\gamma^{N}\in\Gamma$ and weights ${({\alpha}^{1},\dots,{\alpha}^{N})\in\Delta_{N}}$. The resulting $\overline{p}$ is the density of a mixture of probability distributions $\gamma^{n}$ w.r.t. weights ${\alpha}^{n}$.


KL divergence (and the related log-loss) is known to have \textbf{mode seeking property}, i.e. the divergence between $\text{KL}(p_{\omega}||p_{\gamma})$ is small when $p_{\gamma}$ attains huge values in areas where $p_{\omega}$ has huge values. Such behaviour suggests that in some problems KL may not be a good measure of dissimilarity of distributions \cite{regli2018alpha}. Although some other representatives of $f$-divergence class are known to be robust, e.g. Reverse KL-divergence, we do not know whether they are mixable.

\subsubsection{Beta-2 Divergence}
\label{sec-beta-2}
The well-known representative of \textbf{$\beta$-divergences} class \cite{cichocki2010families} used to compare probability distributions is \textbf{Beta-2 divergence}:
$$\mathcal{B}_{2}(p_{\gamma}, p_{\omega})=\int_{\mathcal{X}}|p_{\gamma}(x)-p_{\omega}(x)|^{2}d\mu(x).$$
While KL divergence is known to have \textbf{mode seeking property}, Beta-2 is more \textbf{robust}, see \cite{regli2018alpha}. This property is useful when comparing distributions with outliers.

When $|\mathcal{X}|<\infty$ and $\mu(x)\equiv 1$ for all $x\in\mathcal{X}$, values $p_{\gamma}(x),p_{\omega}(x)$ become probabilities. In this case within the framework of online learning the loss in known as \textbf{Brier score}. In practice it is widely used (similar to CRPS) to assess the quality of weather forecasts \cite{zamo2018estimation}.
Brier score between the distributions on the finite sets is mixable, see e.g. \cite{vovk2009prediction}.

We consider a more general case. Assume that $\|\mu\|_{1}<\infty$ and $\Omega=\Gamma$ is the set of all probability measures $\upsilon$ on $(X,\sigma)$ with $p_{\upsilon}(x)\in [0, M]$ for all $x\in \mathcal{X}$. In this case, the loss function $|p_{\gamma}(x)-p_{\omega}(x)|^{2}$ is $\frac{2}{M^{2}}$-mixable ($\frac{1}{2M^{2}}$-exp-concave) {since it is a squared loss function of inputs $p_{\gamma}(x), p_{\omega}(x)\in[0,M]$}. The divergence $\mathcal{B}_{2}$ is a multiple of an integral squared loss

$$\mathcal{B}_{2}(p_{\gamma}, p_{\omega})=\|\mu\|_{1}\underbrace{\int_{\mathcal{X}}|p_{\gamma}(x)-p_{\omega}(x)|^{2}\frac{1}{\|\mu\|_{1}}d\mu(x)}_{\text{Integral squared loss.}}.$$
We conclude that $\mathcal{B}_{2}$ is $\frac{2}{\|\mu\|_{1}M^{2}}$-mixable ($\frac{1}{2\|\mu\|_{1}M^{2}}$-exp-concave).

For $\frac{1}{2\|\mu\|_{1}M^{2}}$-exp-concave Beta-2 divergence the aggregated prediction $\overline{\gamma}$ is given by its density:
\begin{eqnarray}
\overline{p}(x)=\sum_{n=1}^{N}{\alpha}^{n}\cdot p_{\gamma^{n}}(x)
\label{beta-2-exp-substitution}
\end{eqnarray}
for weights ${({\alpha}^{1},\dots,{\alpha}^{N})\in\Delta_{N}}$ and experts' predictions $\gamma^{1},\dots,\gamma^{N}\in\Gamma$.
The resulting $\overline{p}$ is the density of a mixture of probability distributions $\gamma^{n}$ w.r.t. weights ${\alpha}^{n}$.

For $\frac{2}{\|\mu\|_{1}M^{2}}$-mixable Beta-2 divergence the density of aggregated prediction $\overline{\gamma}$ should satisty:
\begin{equation}
\overline{p}(x)=\frac{M}{2}+\frac{M}{4}\log \frac{\sum_{n=1}^{N}{\alpha}^{n}\exp[-2\big(\frac{M-p_{\gamma^{n}}(x)}{M}\big)^{2}]}{\sum_{n=1}^{N}{\alpha}^{n}\exp[-2\big(\frac{p_{\gamma^{n}}(x)}{M}\big)^{2}]}.
\label{beta-2-pre-substitution}
\end{equation}
Similar to SCRPS in Subsection \ref{sec-scrps}, the result of \eqref{beta-2-pre-substitution} may not be a density function w.r.t. $\mu$. This issue can be solved by the projection trick, i.e. projecting $p_{\overline{\gamma}}$ onto the convex subset of $\mathcal{L}^{2}(\mathcal{X},\mu)$ of $\mathcal{L}^{1}$-integrable non-negative functions which represent densities of distributions w.r.t. $\mu$. The resulting projection will by definition be the density function of some distribution $\overline{\gamma}$. Whereas for finite $\mathcal{X}$ formulas for the projection are tractable and given in e.g. \cite{vovk2009prediction}, in the case of continuous $\mathcal{X}$ analytic formulas for the projection become intractable for arbitrary distributions and finite approximations should be used. Similar to SCRPS, formula \eqref{beta-2-exp-substitution} for the aggregation of predictions is more preferable from the computational point of view.

\subsection{Characteristic Function Discrepancy}
\label{sec-cfd}
Let $\mathcal{X}=\mathbb{R}^{D}$ and assume that $\Gamma=\Omega$ is the set of all probability measures over its Borel field. For a fixed $\sigma$-finite measure $\mu$ on $\mathcal{X}$ and measurable non-negative function $u:\mathcal{X}\rightarrow \mathbb{R}_{+}$ satisfying $\int_{\mathbb{R}^{D}}u(t)d\mu(t)=1$, we consider the \textbf{Characteristic Function Discrepancy} (CFD):
\begin{equation}\text{CFD}_{u,\mu}(\gamma, \omega)=\int_{\mathbb{R}^{D}}\|\phi_{\gamma}(t)-\phi_{\omega}(t)\|^{2}_{\mathbb{C}}u(t)d\mu(t),
\label{cfd}
\end{equation}
where $\phi_{\upsilon}(t)=\mathbb{E}_{x\sim\upsilon}e^{i\langle x,t\rangle}$ denotes the characteristic function of a distribution $\upsilon\in\Gamma$. 
CFD is highly related to \textbf{Maximum Mean Discrepancy} which we discuss in the next Subsection \ref{sec-mmd}, yet in practice it attains faster stochastic computation that was noted by \cite{fatir2019characteristic}. By varying $u$ and $\mu$ it is possible to assign different importances to frequencies of compared probability distributions.

For all $ t\in \mathbb{R}^{D}$ we have $\phi_{\omega}(t),\phi_{\gamma}(t)\in \mbox{Ball}_{\mathbb{C}}(0, 1)$. {In \ref{sec-complex-mix}, we prove that the function ${\lambda: \text{Ball}_{\mathbb{C}}(0, 1)\times \text{Ball}_{\mathbb{C}}(0, 1)\rightarrow \mathbb{R}_{+}}$ defined by 
$$\lambda(z,z')=(z-z')\cdot \overline{(z-z')}=\|z-z'\|^{2}_{\mathbb{C}}.$$
is $\frac{1}{4}$-mixable and $\frac{1}{8}$-exp-concave. Unlike most previous cases, the learning rates for mixability and exp-concavity differ only 2 times but not 4 times. This is due to functions having different domains: $[\mbox{Ball}_{\mathbb{C}}(0, 1)]^{2}\subset \mathbb{C}^{2}$ instead of $[a,b]^{2}\subset\mathbb{R}^{2}$, see \ref{sec-complex-mix} for details. Since CFD is a $\lambda$-loss function, by our Theorem \ref{theorem-main} we conclude that it is also $\frac{1}{4}$-mixable and $\frac{1}{8}$-exp-concave.}


For $\frac{1}{8}$-exp-concave CFD the aggregated prediction $\overline{\gamma}$ is given by its CF
\begin{eqnarray}
\overline{\phi}(x)=\sum_{n=1}^{N}{\alpha}^{n}\cdot \phi_{\gamma^{n}}(x)
\label{cf-exp-substitution}
\end{eqnarray}
for weights ${({\alpha}^{1},\dots,{\alpha}^{N})\in\Delta_{N}}$ and experts' predictions $\gamma^{1},\dots,\gamma^{N}\in\Gamma$. The resulting $\overline{\phi}$ is the CF of a mixture of probability distributions $\gamma^{n}$ w.r.t. weights ${\alpha}^{n}$.

Although we pointed that CFD is $\frac{1}{4}$-mixable, we note that the aggregated result obtained by point-wise function (provided in \ref{sec-complex-mix}) will not be necessarily a CF of some distribution $\overline{\gamma}$. Thus, we have the situation similar to SCRPS and Beta-2 divergence, see Subsections \ref{sec-scrps}, \ref{sec-beta-2}.

\subsection{Maximum Mean Discrepancy}
\label{sec-mmd}
Let $\mathcal{X}=\mathbb{R}^{D}$ and assume that $\Gamma=\Omega$ is the set of all probability measures over its Borel field. Let $k:\mathcal{X}\times\mathcal{X}\rightarrow \mathbb{R}$ be a symmetric positive definite kernel, i.e. for all $N=1,2,\dots$, points $x_{1},\dots,x_{N}\in\mathcal{X}$ and $a_{1},\dots,a_{N}\in \mathbb{R}$ it satisfies
\begin{equation}
    \sum_{n,n'=1}^{N}a_{n}a_{n'} k(x_{n},x_{n'})\geq 0.
    \label{psd-kernel}
\end{equation}
Consider (the square of) \textbf{Maximum Mean Discrepancy} (MMD):
\begin{eqnarray}\text{MMD}_{k}^{2}(\gamma,\omega)=\mathbb{E}_{\gamma\times\gamma}\big[k(x,x')\big]-2\cdot \mathbb{E}_{\gamma\times\omega}\big[k(x,y)\big]+\mathbb{E}_{\omega\times\omega}\big[k(y,y')\big].
\label{mmd-def}
\end{eqnarray}
It is known that \eqref{mmd-def} is non-negative and its root satisfies the triangle inequality. For characteristic kernels $k$ it also holds true that $\text{MMD}_{k}(\gamma,\omega)=0\iff \gamma=\omega$. Thus, for such kernels $\text{MMD}_{k}$ turns to be a metric on the space of probability distributions, see \cite{sriperumbudur2010hilbert} for detailed explanations.

Maximum mean discrepancy is used to compare probability distributions in two-sample hypothesis testing, see \cite{gretton2007kernel}.
Also, useful properties of $\text{MMD}_{k}$, e.g. it naturally admits unbiased estimates from empirical samples, made it widely applicable to generative machine learning, see \cite{li2017mmd}.

In this paper, we consider only symmetric positive definite kernels of the form $k(x,y)=\psi(x-y)$, where $\psi$ is a bounded and continuous function.  Such kernels are usually called \textbf{translation-invariant}.
We note that the majority of kernels used in practice are actually translation-invariant (Gaussian, Laplacian, Sine, etc.), see e.g. Table 2 in \cite{sriperumbudur2010hilbert}.

In the case of translation-invariance, the positive definiteness of kernel $k$ \eqref{psd-kernel} turns to the a positive definiteness of a function $\psi$, i.e. for all $N=1,2,\dots$, points $x_{1},\dots,x_{N}\in\mathcal{X}$ and $a_{1},\dots,a_{N}\in \mathbb{R}$ it holds true that
\begin{equation}
    \sum_{n,n'=1}^{N}a_{n}a_{n'} \psi(x_{n}-x_{n'})\geq 0.
    \label{psd-function}
\end{equation}
According to well-celebrated \textbf{Bochner's Theorem}, see \cite[Theorem 3]{sriperumbudur2010hilbert}, function $\psi$ is positive definite if and only if it is the
Fourier transform of a finite nonnegative Borel measure $\mu$ on $\mathbb{R}^{D}$, that is,
$$\psi(x)=\int_{\mathbb{R}^{D}}e^{-i\langle x,t\rangle}d\mu(t).$$
By using this correspondence between kernel $k$, function $\psi$ and measure $\mu$, \cite[Corollary 4]{sriperumbudur2010hilbert} prove that 
\begin{eqnarray}
\text{MMD}_{k}^{2}(\gamma,\omega)=\int_{\mathbb{R}^{D}}\|\phi_{\gamma}(t)-\phi_{\omega}(t)\|^{2}_{\mathbb{C}}d\mu(t),
\label{mmd-to-cfd}
\end{eqnarray}
where $\phi_{\gamma},\phi_{\omega}$ are the characteristic functions of $\gamma$ and $\omega$ respectively. From \eqref{mmd-to-cfd} we see that $\text{MMD}_{k}$ turns to be a multiple of Characteristic function discrepancy \eqref{cfd}. Indeed, by introducing $u(t)\equiv \frac{1}{\|\mu\|_{1}}$ we obtain
\begin{eqnarray}
\text{MMD}_{k}^{2}(\gamma,\omega)=\|\mu\|_{1}\int_{\mathbb{R}^{D}}\|\phi_{\gamma}(t)-\phi_{\omega}(t)\|^{2}_{\mathbb{C}}u(t)d\mu(t)=
\nonumber
\\
\|\mu\|_{1}\cdot \text{CFD}_{\frac{1}{\|\mu\|_{1}},\mu}(\gamma,\omega),
\label{mmd-to-cfd-final}
\end{eqnarray}
Equation \eqref{mmd-to-cfd-final} immediately means that $\text{MMD}_{k}$ is $\frac{1}{8\|\mu\|_{1}}$-exponentially concave. Analogously to \eqref{cf-exp-substitution}, the aggregated predictive distribution is a mixture of input distributions. We also conclude that $\text{MMD}_{k}$ is $\frac{1}{4\|\mu\|_{1}}$-mixable, although the aggregated prediction is infeasible in general (analogously to CFD, see the discussion of Subsection \ref{sec-cfd}).

\subsection{Optimal Transport Costs}
\label{sec-ot-costs}

\subsubsection{1-Dimensional Optimal Transport}
\label{sec-ot-1d}
Let $\mathcal{X}=\mathbb{R}^{1}$ and assume that $\Gamma=\Omega$ is the space of all probability measures over its Borel field. For a cost function $c:\mathcal{X}\times\mathcal{X}\rightarrow\mathbb{R}$ consider the optimal transport cost \cite{kantorovitch1958translocation} between distributions ${\gamma,\omega\in\Gamma}$:
\begin{equation}
C(\gamma,\omega)=\min_{\mu\in\Pi(\gamma,\omega)}\int_{\Gamma\times\Omega}c(x,x')d\mu(x, x'),    
\label{ot-cost}
\end{equation}
where $\Pi(\gamma,\omega)$ is the set of all probability distributions (transport plans) on $\mathcal{X}\times\mathcal{X}$ whose left and right marginals are distributions $\gamma$ and $\omega$ respectively.

In contrast to the losses considered in Subsection \ref{sec-density-based}, optimal transport cost is defined for arbitrary distributions which may not have densities. Thus, similar to CRPS and CFD, it can be used to compare predicted distribution with discrete outcomes. Besides, optimal transport costs are widely used in many machine learning and image-processing problems, see \cite{peyre2019computational}. In particular, they are applied to generative modeling \cite{arjovsky2017wasserstein,korotin2021wasserstein}.


If cost function is the $p$-th degree of the $p$-th Euclidean norm, i.e. $c(x,x')=\|x-x'\|_{p}^{p}$, the resulting distance $C(\omega,\gamma)$ is called \textbf{Wasserstein-$p$ distance} and denoted by $\mathbb{W}_{p}^{p}$.

In this section, we show that under specific conditions the optimal transport cost is mixable (exp-concave). To begin with, we recall that for $\mathcal{X}=\mathbb{R}$ the optimal transport has a superior property which brings linearity structure to the space of $1$-dimensional probability distributions.
\begin{lemma}[Explicit 1D Optimal Transport]
\label{lemma-ot-1d}
If the transport cost $c:\mathcal{X}\times\mathcal{X}\rightarrow\mathbb{R}$ is twice differentiable and ${\frac{\partial^{2}c(x, x')}{\partial x\partial x'}<0}$, then the optimal transport cost between $\gamma,\omega$ is given by 
\begin{equation}C(\gamma,\omega)=\int_{0}^{1}c\big(\text{\normalfont Q}_{\gamma}(t),\text{\normalfont Q}_{\omega}(t)\big)dt,
\end{equation}
where $\text{\normalfont Q}_{\upsilon}:[0,1]\rightarrow\mathbb{R}$ is the quantile function of $\upsilon\in\Gamma$ defined by 
$${\text{\normalfont Q}_{\upsilon}(t)=\inf\{x\in\mathbb{R}:t\leq \text{\normalfont CDF}_{\upsilon}(x)\}}.$$
\label{ot-quantile-cost}
\end{lemma}
The result was initially proved by \cite{lorentz1953inequality} and then rediscovered several times, see \cite{becker1973theory,mirrlees1971exploration,spence1978job}. The lemma makes it possible to consider the probability distribution's quantile function (inverse cumulative distribution function) instead of the probability distribution for computation of 1-dimensional optimal transport cost. Unlike CRPS (see Figure \ref{fig:crps}), optimal transport compares CDFs not vertically but horizontally, see the illustration in Figure \ref{fig:ot}.

\begin{figure}[!h]
     \centering
     \begin{subfigure}[b]{0.48\columnwidth}
         \centering
         \includegraphics[width=\linewidth]{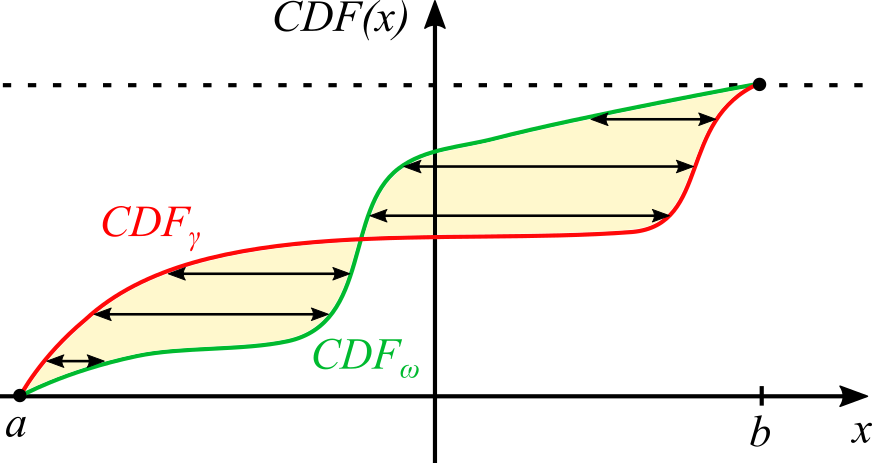}
        \caption{Arbitrary distribution $\omega$.}
     \end{subfigure}
    \begin{subfigure}[b]{0.48\columnwidth}
        \centering
    \includegraphics[width=\linewidth]{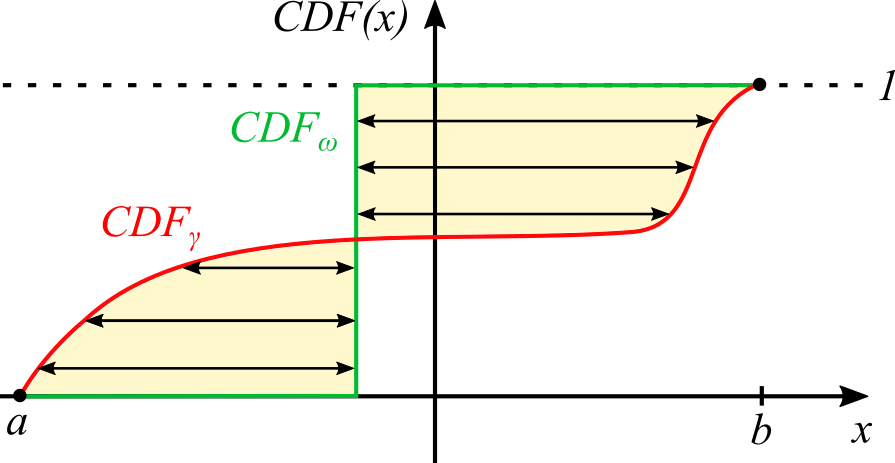}
    \caption{Empirical (Dirac) distribution $\omega$.}
     \end{subfigure}
    \caption{Visualization of the comparison of CDFs of distributions $\gamma,\omega$ on $[a,b]$ by using Optimal Transport Cost.}
    \label{fig:ot}
\end{figure}

From Lemma \ref{lemma-ot-1d} we conclude that if transport cost $c:\mathcal{X}\times \mathcal{X}\rightarrow\mathbb{R}_{+}$ is $\eta$-mixable, then by Theorem \ref{theorem-main} the corresponding $C(\gamma,\omega)$ is $\eta$-mixable (exp-concave) $\mathcal{X}$-integral $c$-loss function.

For $\eta$-mixable transport cost the aggregated prediction's quantile function should satisfy:

$$\overline{\text{Q}}(t)=\Sigma_{c}\big(\{\text{Q}_{\gamma^{n}}(t),{\alpha}^{n}\}_{n=1}^{N}\big).$$
for experts' predictions $\gamma^{1},\dots,\gamma^{N}\in\Gamma$ and weights ${({\alpha}^{1},\dots,{\alpha}^{N})\in\Delta_{N}}$. In general, $\overline{\text{Q}}(t)$ may not be a quantile function of some distribution. It is necessarily a quantile function if the substitution $\Sigma_{c}$ for transport cost $c$ is monotone. However, even if $\Sigma_{c}$ is not monotone, $\overline{\text{Q}}(t)$ can be used to implicitly model the aggregated prediction. To show it, we state and prove

\begin{lemma}[Implicit Aggregation for 1D Optimal Transport]
Let $\zeta$ be the uniform probability measure on $[0,1]$. Consider the pushforward probability measure $\overline{\gamma}=\overline{\text{Q}}\circ \zeta$, i.e. $\overline{\gamma}$ is distributed according to $\overline{\text{Q}}(t)$ for $t\sim \zeta$. Then for all $\omega\in\Omega$ it satisfies:
\begin{equation}\exp\big(-\eta C(\overline{\gamma}, \omega)\big)\geq \sum_{n=1}^{N}{\alpha}^{n}\exp\big[-\eta C(\gamma^{n},\omega)\big].
\label{implicit-lemma-expr}
\end{equation}
\label{lemma-implicit}
\end{lemma}
Lemma \ref{lemma-implicit} explains the mixability of the optimal transport cost $C$ with arbitrarily substitution $\Sigma_{c}$. Additionally, it provides a natural way to sample from the aggregated distribution $\overline{\gamma}$, i.e. sampling from $t\sim\text{Uniform}[0,1]$ and applying $\overline{\text{Q}}$. Now we prove the lemma.

\begin{proof}
Consider any $\omega\in\Omega$ and let $\omega\in\Pi(\overline{\gamma},\omega)$ be a transport plan (between $\overline{\gamma}$ and $\omega$) given by $\omega=[\overline{\text{Q}}, \text{Q}_{\omega}]\circ \zeta=[\overline{\text{Q}}\circ \zeta, \text{Q}_{\omega}\circ \zeta]$. Since $\omega$ is not necessarily the optimal transport plan, its cost is not smaller then the cost of the optimal one (which is given by Lemma \ref{lemma-ot-1d}):
\begin{eqnarray}\int_{0}^{1}c(\overline{\text{Q}}(t),\text{Q}_{\omega}(t))dt=\int_{\mathcal{X}\times\mathcal{X}}c(x,x')d\omega(x,x')\geq
\nonumber
\\
\int_{0}^{1}c(\text{Q}_{\overline{\gamma}}(t),\text{Q}_{\omega}(t))dt=C(\overline{\gamma},\omega),
\label{non-optimal-plan}
\end{eqnarray}
where $\text{Q}_{\overline{\gamma}}$ is the quantile function of $\overline{\gamma}$. Since $\overline{\text{Q}}$ is obtained by the substitution function $\Sigma_{c}$, we have
\begin{equation}\exp\big[-\eta \int_{0}^{1}c(\overline{\text{Q}}(t),\text{Q}_{\omega}(t))dt\big]\geq \sum_{n=1}^{N}{\alpha}^{n}\exp\big[-\eta C(\gamma^{n},\omega)\big].
\label{mixability-implicit-quantile}
\end{equation}
We combine \eqref{non-optimal-plan} with \eqref{mixability-implicit-quantile} and obtain desired \eqref{implicit-lemma-expr}.
\end{proof}

For $\eta$-exp-concave cost the aggregated prediction's quantile function is given by:
$$\overline{\text{Q}}(t)=\sum_{n=1}^{N}{\alpha}^{n}\cdot \text{Q}_{\gamma^{n}}(t).$$
The obtained function (for all admissible costs $c$!) is a quantile function of a \textbf{Wasserstein-2 barycenter} $\overline{\gamma}$ of distributions $\gamma^{1},\dots,\gamma^{N}$ w.r.t. weights ${({\alpha}^{1},\dots,{\alpha}^{N})\in\Delta_{N}}$, see \cite[Corollary 1]{bonneel2015sliced}.

\subsubsection{Sliced Wasserstein-2 Distance}
\label{sec-sliced-ot}

Definition \eqref{ot-cost} has a natural multidimensional extension to $\mathcal{X}=\mathbb{R}^{D}$. However, in this case the resulting optimal transport cost does not admit representation analogous to the one provided in Lemma \ref{lemma-ot-1d}. Even for the squared cost $c(x,x')=\|x-x'\|^{2}_{2}$ the Wasserstein-2 metric space of distributions is highly non-linear and has negative curvature.

Instead, we prove that \textbf{Sliced Wasserstein-2 distance} might be mixable under certain conditions. Following \cite{bonneel2015sliced}, we assume that $\Gamma=\Omega$ is the set of Radon (locally finite Borel) probability measures on $\mathcal{X}\subset \mathbb{R}^{D}$ with finite second moment. For all $\upsilon\in\Gamma$ we define sliced quantile function by 
$$\text{SQ}(\theta,t)=\inf\{s\in\mathbb{R}:t\leq \text{SCDF}_{\upsilon}(\theta,s)\}.$$
Next, we define sliced quadratic transport cost (the square of Sliced $\mathbb{W}_{2}$ distance):
\begin{equation}\text{S}\mathbb{W}_{2}^{2}(\gamma,\omega)=\frac{1}{S_{D-1}}\int_{\mathbb{S}^{D-1}}\bigg[\int_{0}^{1}\big(\text{SQ}_{\gamma}(\theta,t)-\text{SQ}_{\omega}(\theta,t)\big)^{2}dt\bigg]d\theta.
\label{scost-definition}
\end{equation}
From the definition we see that $\text{S}\mathbb{W}_{2}^{2}$ is the average over all the sliced quadratic transport costs. If $\gamma,\omega$ have supports $\subset \text{Ball}_{\mathbb{R}^{D}}(0,R)$, then $\text{SQ}(\theta,t)\in[-R,R]$ and point-wise squared loss is $(2R^{2})^{-1}$-mixable ($(8R^{2})^{-1}$-exp-concave). We use Theorem \ref{theorem-main} and (similar to CRPS) conclude that $\text{S}\mathbb{W}^{2}_{2}$ is $(2R^{2})^{-1}$-mixable ($(8R^{2})^{-1}$-exp-concave).

For $(8R^{2})^{-1}$-exp-concave sliced cost the aggregated prediction's sliced quantile should satisfy:
\begin{equation}\overline{\text{SQ}}(\theta,t)=\sum_{n=1}^{N}{\alpha}^{n}\cdot \text{SQ}_{\gamma_{n}}(\theta,t)
\label{sc-exp-substitution}
\end{equation}
for experts' predictions $\gamma^{1},\dots,\gamma^{N}\in\Gamma$ and weights ${({\alpha}^{1},\dots,{\alpha}^{N})\in\Delta_{N}}$.

Unfortunately, for $D>1$ function $\overline{\text{SQ}}(\theta,t)$ \textbf{is not} necessarily a quantile function of some distribution, see the discussion in \cite{bonneel2015sliced}. In particular, the image of the map $\gamma\mapsto SQ_{\gamma}$ is not necessarily convex (as a subset of the space $\mathcal{L}^{2}([0,1])$). As a corollary, we see that the projection trick (the one we used for other losses, see e.g. Subsection \ref{sec-scrps}) is \textbf{not} applicable due to the mentioned non-convexity.

We recall the \textbf{sufficient condition} \cite{bonneel2015sliced} for $\overline{\text{SQ}}(\theta,t)$ being a sliced quantile of some distribution. For probability measure $\gamma$ and $(s,u)\in\mathbb{R}_{+}\times \mathbb{R}^{D}$ we use $\psi_{s,u}\circ\gamma$ to denote a probability measure obtained from $\gamma$ by pushing it forward with $\psi_{s,u}(x)=sx+u$, i.e. $\psi_{s,u}\circ\gamma$ is a scaled and translated copy of $\gamma$.

\begin{lemma}[Barycenter of scaled and translated distributions]
Assume that all the predictions $\gamma^{1},\dots,\gamma^{N}$ are scaled and translated by $\psi_{s^{1},u^{1}},\dots,\psi_{s^{N},u^{N}}$ copies of some reference probability measure $\gamma^{0}$. Then aggregated $\overline{\text{SQ}}$ (given in \eqref{sc-exp-substitution}) is the sliced quantile of $\psi_{\overline{s},\overline{u}}\circ\gamma^{0}$, where 
$$\overline{s}=\big(\sum_{n=1}^{N}\frac{{\alpha}^{n}}{s^{n}}\big)^{-1}\qquad\text{and}\qquad\overline{u}=\frac{\sum_{n=1}^{N}\frac{{\alpha}^{n}u^{n}}{s^{n}}}{\sum_{n=1}^{N}\frac{{\alpha}^{n}}{s^{n}}}.$$
Measure $\psi_{\overline{s},\overline{u}}\circ\gamma^{0}$ is called \textbf{Sliced Wasserstein-2 Barycenter} of $\gamma^{1},\dots,\gamma^{N}$ w.r.t. weights ${\alpha}^{1},\dots,{\alpha}^{N}$.
\label{lemma-scale-rot-bar}
\end{lemma}
The proof of the fact is given in \cite[Proposition 12]{bonneel2015sliced}. Lemma \ref{lemma-scale-rot-bar} provides an explicit formula to compute the aggregated prediction when all the predictive distributions are scaled and translated copies of each other.

We do not now whether the \textbf{necessary condition} for $\overline{\text{SQ}}$ being a quantile of some distribution exists. We also leave the question for existence of the aggregated prediction for $(2R^{2})^{-1}$-mixability property open (even in the case of scaled and translated copies that we considered for exp-concavity). These open questions serve as the challenge for our further research.

\section{Conclusion}

{In this paper, we proved that natural integral losses arising from (one dimensional) mixable (exp-concave) losses are also mixable (exp-concave).} As a consequence of our main result, we demonstated that a wide range of losses for comparing the probability distributions are indeed mixable (exp-concave). In most cases the substitution function for mixability is computationally heavy and practically inaplicable. Yet for exp-concavity the aggregated prediction simply reduces to computation of the mixture of predictions or Wasserstein-2 barycenter of predictions.

Our theoretical results indicate that a constant-regret online prediction with experts' advice with \textbf{probabilistic} forecasts and outcomes is possible. Providing high-quality online probabilistic forecasts is essential for various applications of machine learning including electricity consumption forecasting, weather forecasting, stock prices prediction, etc. We believe that our findings will help to improve existing models for prediction  and obtain stronger theoretical justification of their performance.

\section*{Acknowledgements}
\noindent The work was supported by the Russian Foundation for Basic Research grant 21-51-12005 NNIO\_a.





\appendix

\section{Mixability and Exp-Concavity of Complex Squared Loss}
\label{sec-complex-mix}
\noindent In this section we prove $\frac{1}{4}$-mixability and $\frac{1}{8}$-exp-concavity for 
$${\lambda:\text{Ball}_{\mathbb{C}}(0,1)\times \text{Ball}_{\mathbb{C}}(0,1)\rightarrow \mathbb{R}_{+}}$$
given by $\lambda(z,z')=\|z-z'\|^{2}_{\mathbb{C}}$. First, we recall well-known lemma \cite[Lemma 4.2]{hazan2016introduction} on exponential concavity:

\begin{lemma}Let $\Gamma\subset\mathbb{R}^{D}$ be a convex set. Let $\lambda:\Gamma\times\Omega\rightarrow\mathbb{R}$ be a loss function which is twice diffitentiable over $\gamma$ for all $\omega\in\Omega$. Then $\lambda$ is $\eta$-exponentially concave (for $\eta>0$) iff for all $\omega\in\Omega$:
\begin{equation}\nabla^{2}_{\gamma}\lambda(\gamma, \omega)\succeq \eta \cdot [\nabla_{\gamma}\lambda(\gamma, \omega)]\cdot [\nabla_{\gamma}\lambda(\gamma, \omega)]^{\top}
\label{exp-concavity-condition}\end{equation}
\end{lemma}
\begin{proof}The lemma follows directly from concavity condition of functions ${f_{\omega}(\gamma)=e^{-\eta\cdot \lambda(\gamma,\omega)}}$ for all $\omega\in\Omega$. For $\omega\in\Omega$ the function $f_{\omega}(\gamma)$ is concave in $\gamma$ iff
$$0\succeq\nabla_{\gamma}^{2}f_{\omega}(\gamma)\Leftrightarrow 0\succeq \big(-\eta \nabla^{2}_{\gamma}\lambda(\gamma, \omega)+\eta^{2}\cdot [\nabla_{\gamma}\lambda(\gamma, \omega)]\cdot [\nabla_{\gamma}\lambda(\gamma, \omega)]^{\top}\big)\cdot e^{-\eta\lambda(\gamma,\omega)}$$
which is equivalent to \eqref{exp-concavity-condition} due to positivity of exponent and $\eta>0$.\end{proof}

\begin{corollary}The function $\lambda(\gamma,\omega)=\|\gamma-\omega\|^{2}$ is $\frac{1}{8B^{2}}$-exponentially concave in the first argument for $\omega,\gamma\in \Omega=\Gamma=\mathbb{B}_{\mathbb{R}^{D}}(0, B)$.
\end{corollary}
\begin{proof}
We check that $\lambda$ satisfies the condition \eqref{exp-concavity-condition}, i.e.
$$2I_{D}\succeq \frac{1}{2 B^{2}}(\gamma-\omega)\cdot (\gamma-\omega)^{T},$$
where $I_{D}$ is the $D$-dimensional identity matrix. The inequality holds true for all $\omega, \gamma$ because the value on the right is a positive semi-definite matrix with the only non-zero eigenvalue equal to $\|\gamma-\omega\|^{2}\leq (2B)^{2}=4B^{2}$.
\end{proof}

Due to natural mapping $z\mapsto (\text{Re } z,\text{Im } z)$, function $\lambda(z,z')=\|z-z'\|^{2}_{\mathbb{C}}$ can be viewed as a squared loss on $\mathbb{R}^{2}$, i.e. ${\lambda:\text{Ball}_{\mathbb{R}^{2}}(0,1)\times \text{Ball}_{\mathbb{R}^{2}}(0,1)\rightarrow \mathbb{R}_{+}}$. Thus, from corollary we conclude that $\lambda$ it $\frac{1}{8}$-\textbf{exponentially concave} on a unit ball.

To prove $\frac{1}{4}$-mixability we also view the function as a function on $\mathbb{R}^{2}\times\mathbb{R}^{2}$ and note that $\text{Ball}_{\mathbb{R}^{2}}(0,1)\subset [-1,1]^{2}$. Thus, the function can be viewed as the vectorized squared loss \eqref{vector-square-loss} with $D=2$, multiplied by $2$. Since vectorized squared loss with inputs from $[-1,1]^{D}$ is $\frac{1}{2}$-mixable, we conclude that $\lambda$ is $\frac{1}{4}$-mixable with the substitution function given by
$$\Sigma_{\lambda}\big(\{\gamma^{n},{\alpha}^{n}\}_{n=1}^{N}\big)=\bigg(\Sigma_{L^{2}}^{[-1,1]}\big(\{\text{Re }\gamma^{n},{\alpha}^{n}\}_{n=1}^{N}\big),\Sigma_{L^{2}}^{[-1,1]}\big(\{\text{Im }\gamma^{n},{\alpha}^{n}\}_{n=1}^{N}\big)\bigg)$$
for weights ${({\alpha}^{1},\dots,{\alpha}^{N})\in\Delta_{N}}$ and experts' predictions $\gamma^{1},\dots,\gamma^{N}\in\mathbb{C}$, see definition of $\Sigma_{L^{2}}^{[l,r]}$ in \eqref{squared-loss-substitution}.

\section{Aggregating Algorithm}
\label{sec-aa}

In this section for completeness of the exposition we review the strategy of the aggregating algorithm by \cite{vovk1998game}.
Besides, we recall the basic analysis required to establish constant regret bound for AA.

The algorithm by \cite{vovk1998game} keeps and updates a weight vector $${({\alpha}^{1}_{t},{\alpha}^{2}_{t},\dots,{\alpha}^{N}_{t})\in \Delta_{N}}$$ which estimates the performance of the experts in the past. These weights are used to construct a combined forecast from experts predictions' by using the substitution function $\Sigma_{\lambda}$. The weights of experts are used to combine the forecast for the step $t$ simply proportional to exponentiated version of experts' cumulative losses, i.e. 
$${\alpha}_{t}^{n}\propto \exp[-\eta L_{t}^{n}]= \exp[-\eta\sum_{\tau=1}^{t-1}l_{t}^{n}]=\exp[-\eta\sum_{\tau=1}^{t-1}\lambda(\gamma_{\tau}^{n},\omega_{\tau}^{n})].$$
The Aggregating Algorithm \ref{vovk-aa} is shown below. Next, we review the analysis of the algorithm and recall the importance of mixability (exp-concavity).
\begin{algorithm}[!h]
\SetAlgorithmName{Algorithm}{empty}{Empty}
\SetKwInOut{Parameters}{Parameters}
\Parameters{Pool of experts $\mathcal{N}=\{1,2,3\dots,N\}$; Game length $T$; $\eta$-mixable loss function $\lambda:\Gamma\times \Omega\rightarrow \mathbb{R}$}
The learner sets experts' weights ${\alpha}_{1}^{1},{\alpha}_{1}^{2},\dots,{\alpha}_{1}^{N}\equiv \frac{1}{N}$\;
\For{$t=1,2,\dots,T$}{
  1. Experts $n\in\mathcal{N}$ provide forecasts $\gamma_{t}^{n}\in \Gamma$\;
  2. Algorithm combines forecasts of experts into single forecast $$\gamma_{t}=\Sigma_{\lambda}(\gamma^{1}_{t},{\alpha}^{1}_{t},\dots, \gamma^{N}_{t},{\alpha}^{N}_{t})\;$$
  3. Nature reveals true outcome $\omega_{t}\in\Omega$\;
  4. Experts $n\in\mathcal{N}$ suffer losses $l_{t}^{n}=\lambda(\gamma_{t}^{n},\omega_{t})$\;
  5. The learner suffers loss $h_{t}=\lambda(\gamma_{t},\omega_{t})$\;
  6. The learner updates and normalizes weights for the next step:
  $${\alpha}_{t+1}^{n}=\frac{{\alpha}_{t}^{n}\exp[-\eta l_{t}^{n}]}{\sum_{n'=1}^{N}{\alpha}_{t}^{n'}\exp[-\eta l_{t}^{n'}]};$$
 }
\caption{Online Aggregating Algorithm}
\label{vovk-aa}
\end{algorithm}

\noindent First, we define the notion of \textbf{mixloss} at the step $t$:
$$m_{t}=-\frac{1}{\eta}\log \sum_{n=1}^{N}{\alpha}_{t}^{n}\exp[-\eta l_{t}^{n}].$$ Note that since aggregating of forecasts is performed via the substitution function, we have 
$$\exp[-\eta h_{t}]=\exp[-\eta \lambda(\overline{\gamma_{t}}, \omega_{t})]\geq \sum_{n=1}^{N}{\alpha}^{n}_{t}\exp[-\eta \underbrace{\lambda(\gamma_{t}^{n},\omega_{t})}_{l_{t}^{n}}]=\exp[-\eta m_{t}],$$
thus, $h_{t}\leq m_{t}$. Next, we denote the \textbf{cumulative mixloss} by
$$M_{t}=-\frac{1}{\eta}\log \sum_{n=1}^{N}{\alpha}_{t}^{n}\exp[-\eta L_{t}^{n}].$$
One may check by a direct computation that $M_{t}-M_{t-1}=m_{t}$ for all $t$. Usual analysis shows that
\begin{eqnarray}H_{T}=\sum_{t=1}^{T}h_{t}\leq \sum_{t=1}^{T}m_{t}=\sum_{t=1}^{T}(M_{t}-M_{t-1})=M_{T}=-\frac{1}{\eta}\log \sum_{n=1}^{n}{\alpha}_{1}^{n}\exp[-\eta L_{t}^{n}]=
\nonumber
\\
-\frac{1}{\eta}\log \frac{1}{N} \sum_{n=1}^{N}\exp[-\eta L_{t}^{n}]=\frac{\log N}{\eta}-\log \sum_{n=1}^{N}\exp[-\eta L_{t}^{n}]\leq 
\frac{\log N}{\eta} + \min_{n\in\mathcal{N}}L_{T}^{n},
\nonumber
\end{eqnarray}
which exactly means that regret w.r.t. the best expert does not exceed $\frac{\ln N}{\eta}$.

AA (as an \textbf{exponential weights} algorithm) has a number of Bayesian probabilistic interpretations \cite{adamskiy2012putting,KOROTIN2019} and other modifications \cite{herbster1998tracking,adamskiy2013adaptive,korotin2020online}.
These extensions are not discussed in this paper. They are follow ups of the algorithm rather than the loss function we are mostly interested in. But we note that our framework {of mixability (exp-concavity) of integral losses} by default admits all these described extensions.


%
%

\bibliographystyle{abbrv}      
\bibliography{references}   

\end{document}